\documentclass[accepted]{uai2021} 

\usepackage[american]{babel}

\usepackage{natbib} 
\usepackage{mathtools} 
\usepackage{booktabs} 
\usepackage{tikz} 

\title{Variance-Dependent Best Arm Identification}

  \newtheorem{example}{Example}
  \newtheorem{theorem}{Theorem}
  \newtheorem{lemma}[theorem]{Lemma}
  \newtheorem{proposition}[theorem]{Proposition}
  
  \newtheorem{corollary}[theorem]{Corollary}

\newcommand{\ignore}[1]{}

\usepackage{algorithm}
\usepackage{algorithmic}
\usepackage[algo2e]{algorithm2e}
\LinesNumbered

\usepackage{chngcntr}
\counterwithin*{theorem}{section}
\renewcommand{\thetheorem}{\arabic{section}.\arabic{theorem}}

\usepackage{fullpage}
\usepackage{cleveref}
\usepackage{enumitem}
\usepackage{amssymb}
\usepackage{comment}
\usepackage{listings}
\usepackage{multirow}
\usepackage{multicol}
\usepackage{dsfont}
\usepackage{mathtools}

\usepackage{xr}

\newlist{thmlist}{enumerate}{1}
\setlist[thmlist]{label=(\alph{thmlisti}), ref=\thetheorem(\alph{thmlisti}),noitemsep}
\newlist{lemlist}{enumerate}{1}
\setlist[lemlist]{label=(\alph{lemlisti}), ref=\thetheorem(\alph{lemlisti}),noitemsep}

\usepackage{mathrsfs}
\usepackage{xcolor}
\setenumerate[0]{label=(\alph*)}
\PassOptionsToPackage{hyphens}{url}\usepackage{hyperref}
\newenvironment{myproof}[1]{%
  {\flushleft \textbf{#1}\quad }%
}{\endproof}

  \newenvironment{proof}%
  {%
   \par\noindent{\bfseries\upshape Proof\ }%
  }%

\DeclareMathOperator*{\polylog}{\mathrm{polylog}}

\let\hat\relax
\newcommand{\hat}{\widehat}

\let\tilde\relax
\newcommand{\tilde}{\widetilde}
\let\leq\relax
\newcommand{\leq}{\leqslant}

\let\geq\relax
\newcommand{\geq}{\geqslant}

\newcommand{\eps}{\epsilon}
\DeclareMathOperator*{\E}{\mathbb{E}}
\DeclareMathOperator*{\Var}{\mathrm{Var}}
\let\Pr\relax
\let\hat\widehat
\DeclareMathOperator*{\Pr}{\mathrm{Pr}}

\DeclareMathOperator*{\argmax}{\mathrm{argmax}}

\newcommand{\mathsc}[1]{{\normalfont\textsc{#1}}}

\newcommand{\MeanEst}{\mathsc{MeanEst}}

\newcommand{\VarEst}{\mathsc{VarEst}}
\newcommand{\VarTest}{\mathsc{VarTest}}
\newcommand{\NaiveOptArmId}{\mathsc{NaiveBestArm}}
\newcommand{\NaiveOptArmEst}{\mathsc{NaiveBestArmEst}}
\newcommand{\OptArmEst}{\mathsc{BestArmEst}}
\newcommand{\OptArmId}{\mathsc{VD\hbox{-}BestArmId}}
\newcommand{\OptArmIdE}{\mathsc{VD\hbox{-}BestArmId*}}
\newcommand{\HE}{\mathsc{GroupElim}}
\newcommand{\ME}{\mathsc{IterElim}}

\newcommand{\KL}{\mathrm{KL}}

\newcommand{\hB}{\widehat{B}}

\newcommand{\bbA}{\mathbb{A}}

\newcommand{\calE}{\mathcal{E}}
\newcommand{\calF}{\mathcal{F}}
\newcommand{\calM}{\mathcal{M}}

\newcommand{\true}{{\sf true}\xspace}
\newcommand{\false}{{\sf false}\xspace}

\newcommand*\samethanks[1][\value{footnote}]{\footnotemark[#1]}

\author[1]{Pinyan Lu\thanks{Authors are listed in alphabetical order}}
\author[2]{Chao Tao\samethanks}
\author[3]{Xiaojin Zhang\samethanks}
\affil[1]{%
    ITCS\\
    Shanghai University of Finance and Economics
}
\affil[2]{%
    Department of Computer Science\\
    Indiana University Bloomington
}
\affil[3]{%
    Department of Computer Science and Engineering\\
    The Chinese University of Hong Kong
}

\begin{document}
\maketitle

\begin{abstract}
We study the problem of identifying the best arm in a stochastic multi-armed bandit game. Given a set of $n$ arms indexed from $1$ to $n$, each arm $i$ is associated with an unknown reward distribution supported on $[0,1]$ with mean $\theta_i$ and variance $\sigma_i^2$. Assume $\theta_1 > \theta_2 \geq \cdots \geq\theta_n$. We propose an adaptive algorithm which explores the gaps and variances of the rewards of the arms and makes future decisions based on the gathered information using a novel approach called \textit{grouped median elimination}. The proposed algorithm guarantees to output the best arm with probability $(1-\delta)$ and uses at most $O \left(\sum_{i = 1}^n \left(\frac{\sigma_i^2}{\Delta_i^2} + \frac{1}{\Delta_i}\right)(\ln \delta^{-1} + \ln \ln \Delta_i^{-1})\right)$ samples, where $\Delta_i$ ($i \geq 2$) denotes the reward gap between arm $i$ and the best arm and we define $\Delta_1 = \Delta_2$. This achieves a significant advantage over the variance-independent algorithms in some favorable scenarios and is the first result that removes the extra $\ln n$ factor on the best arm compared with the state-of-the-art. We further show that $\Omega \left( \sum_{i = 1}^n \left( \frac{\sigma_i^2}{\Delta_i^2} + \frac{1}{\Delta_i} \right) \ln \delta^{-1} \right)$ samples are necessary for an algorithm to achieve the same goal, thereby illustrating that our algorithm is optimal up to doubly logarithmic terms.
\end{abstract}

\section{Introduction} \label{sec:intro}
The stochastic multi-armed bandit (MAB) is a famous framework that captures well the trade-off between exploration and exploitation. In the MAB game, a player faces a set of $n$ ($n \geq 2$) arms indexed from $1$ to $n$. When arm $i$ is sampled, the player observes an instant reward which is \textit{i.i.d.}\ generated from an unknown distribution $\mathcal{D}_i$ supported on $[0, 1]$ with mean $\theta_i$ and variance $\sigma_i^2$. In the \emph{pure exploration} setting of a MAB game, by making a sequence of samples, the player identifies one (or a set of) desired arm(s). This framework is motivated by many application domains such as medical trials \cite{robbins1952some}, communication networks \cite{audibert2010best}, simulation optimization \cite{chen2011stochastic}, recommendation systems \cite{kohli2013fast}, and crowdsourcing \cite{zhou2014optimal}.

In this paper, we focus on the \emph{best arm identification} problem. The \emph{best arm} is the one with the maximum expected reward. Without loss of generality, we assume $\theta_1 > \theta_2 \geq \cdots \geq \theta_n$ which is however not known beforehand to the player. We say an algorithm is \emph{$\delta$-correct} if it returns the best arm with probability at least $(1-\delta)$. The goal of the best arm identification problem is to design an algorithm equipped by the player to $\delta$-correctly identify the best arm, with as few samples as possible. Previously, the confidence intervals were mainly constructed utilizing the mean rewards of the arms, e.g., \cite{even2002pac, audibert2010best, gabillon2012best, karnin2013almost, jamieson2014lil, chen2015optimal}. It is worth noting that the variance of the rewards also embodies important information. The variance of rewards could be employed to provide significant advantages over the pure mean-based algorithms. We design an efficient algorithm to solve the problem of best arm identification by exploiting the variance of the rewards, which requires significantly fewer samples in many favorable cases. We further provide a lower bound which illustrates that our algorithm is optimal up to doubly logarithmic terms.

\subsection{Related Works}\label{sec:related}
In the seminal work of \cite{even2002pac}, the authors showed that if $\theta_{1} - \theta_{2} \geq \Delta$, then their Median Elimination algorithm uses at most $O( \frac{n} {\Delta^2 } \ln \delta^{-1} )$ samples \footnote{In fact, the algorithm provides the following stronger (PAC) guarantee -- if there are multiple arms with mean rewards at least $(\theta_1 - \Delta)$, then the algorithm returns an arbitrary one among these arms.}. In the same paper, they also showed that for every $\delta$-correct algorithm, the worst-case sample complexity among all instances such that $\theta_1 - \theta_2 \geq \Delta$ is at least $\Omega( \frac{n} {\Delta^2 } \ln \delta^{-1})$. The $\Theta( \frac{n} {\Delta^2 } \ln \delta^{-1})$ bound can be improved when the input data is easy, which is measured via the \emph{reward gaps} between every sub-optimal arm and the best arm. Formally,  let $\Delta_{i} = \theta_{1} - \theta_{i}$ for $i \geq 2$ and $\Delta_{1} = \Delta_2$ denote the reward gaps. Intuitively, less samples are required if many reward gaps are significantly larger than $\Delta = \Delta_1$. With this intuition, \cite{even2002pac} showed the first \emph{gap-dependent} algorithm called Successive Elimination, which achieves $\delta$-correctness using $O( \sum_{i=2}^n \Delta_i^{-2} (\ln \delta^{-1} + \ln n + \ln \ln \Delta_i^{-1}))$ samples. Since then, the gap-dependent algorithms for the best arm identification problem have been extensively studied, e.g., \cite{gabillon2012best, karnin2013almost, jamieson2014lil, chen2015optimal, chen2017towards}. Both the Exponential Gap Elimination algorithm in \cite{karnin2013almost} and the lil'UCB algorithm in \cite{jamieson2014lil} achieve $\delta$-correctness with sample complexity  \footnote{Here for simplicity we assume $\Delta_i$ is sufficiently small, and the same applies to the rest of this paper. When $\Delta_i$ approaches $1$, the doubly logarithmic term should be $\ln (e + \ln \Delta_i^{-1})$ to avoid negative evaluations.}
\begin{align} \label{eq:gap-dependent-bound-1}
O \left(\sum_{i = 2}^n \Delta_i^{-2} (\ln \delta^{-1} + \ln \ln \Delta_i^{-1})\right).
\end{align}
\cite{chen2017towards} further showed a $\delta$-correct algorithm with sample complexity
\begin{multline} \label{eq:gap-dependent-bound-2}
O \Bigg(\sum_{i = 2}^n \frac{1}{\Delta_{i}^2} (\ln \delta^{-1} + \mathrm{Ent}(\Delta_{2}, \dots, \Delta_{n})) \\ 
 + \frac{1}{\Delta_{2}^2} \ln \ln \frac{1}{\Delta_{2}} \cdot \polylog (n, \delta^{-1}) \Bigg),
\end{multline}
where $\mathrm{Ent}(\Delta_{2}, \dots, \Delta_{n})$ is an entropy-like function. This bound improves the result of \cite{karnin2013almost} and \cite{jamieson2014lil} when the second additive term is dominated by the first term (which is the usual case).

On the lower bound side, \cite{mannor2004sample, kaufmann2016complexity} showed that every gap-dependent $\delta$-correct algorithm uses at least $\Omega( \sum_{i = 2}^n \Delta_i^{-2} \ln \delta^{-1})$ samples in expectation; and this lower bound holds for all possible gap parameters. Based on the results in \cite{farrell1964asymptotic}, \cite{jamieson2014lil} showed that even when there are only two arms, for every $0.1$-correct algorithm, there exists an input instance where $\Omega(\Delta^{-2} \ln \ln \Delta^{-1})$ samples are needed. Therefore the sample complexity in  \eqref{eq:gap-dependent-bound-1} matches the lower bound up to $\ln \ln \Delta_i^{-1}$ terms for $i \geq 3$. The first mentioned lower bound was further improved by \cite{chen2017towards} to $\Omega(\sum_{i =2}^n \Delta_i^{-2} (\ln \delta^{-1} + \mathrm{Ent}(\Delta_{2}, \dots, \Delta_{n}))$.

To further improve the sample complexity, another line of research tries to leverage information beyond reward gaps i.e., variance \cite{gabillon2012best} and Kullback–Leibler (KL) divergence \cite{DBLP:journals/jmlr/MaillardMS11,DBLP:journals/jmlr/GarivierC11,DBLP:conf/colt/KaufmannK13,DBLP:conf/nips/TanczosNM17} to construct a more refined confidence interval. Let $\KL(X, Y)$ denote the KL-divergence between two random variables $X$ and $Y$. The state-of-the-art algorithm lil-KLUCB proposed in \cite{DBLP:conf/nips/TanczosNM17} utilizes Chernoff information, derived from the KL divergence and achieves a high-probability sample complexity upper bound scaling as
\begin{align*}
\inf_{\tilde{\theta}_2, \dots, \tilde{\theta}_n} \frac{1}{ \mathrm{D}^*(\theta_1, \tilde{\theta})} \left( \ln(n/\delta) + \ln \ln \frac{1}{ \mathrm{D}^*(\theta_1, \tilde{\theta})}  \right) \\
+ \sum_{i \geq 2} \frac{1}{ \mathrm{D}^*(\theta_i, \tilde{\theta}_i)} \left( \ln \delta^{-1} + \ln \ln \frac{1}{ \mathrm{D}^*(\theta_i, \tilde{\theta}_i)}  \right),
\end{align*}

where $\tilde{\theta}_i \in (\theta_i, \theta_1)$, $\tilde{\theta} = \max_{i \geq 2} \tilde{\theta}_i$, and $\mathrm{D}^*(x, y) = \max_{z \in [x, y]} \min \{ \KL(\mathrm{Ber}(z), \mathrm{Ber}(x)), \KL(\mathrm{Ber}(z), \mathrm{Ber}(y)) \}$ denotes the Chernoff information. However, there is still a $\ln n$ factor appearing in the term corresponding to the number of samples on the best arm.

\subsection{Our Results}

\begin{theorem}[Restatement of Theorem~\ref{thm:opt-arm-id}]
We propose an algorithm called $\OptArmId(n, \delta)$ which, with probability at least $(1-\delta)$, outputs the best arm and uses at most
\begin{align}\label{eq:var-dependent-bound}
O \left( \sum_{i = 1}^n \left( \frac{\sigma_{i}^2}{\Delta_{i}^2} + \frac{1}{\Delta_i} \right) ( \ln \delta^{-1} + \ln \ln \Delta_i^{-1} ) \right)
\end{align}
samples.
\end{theorem}

Since the expected sample complexity of $\OptArmId$ is not guaranteed to be bounded, using the trick developed in \cite{chen2017towards}, we are also able to transform $\OptArmId$ to an algorithm whose expected sample complexity is bounded.

\begin{theorem} \label{THM:EXPECTED-TIME-ALGO}
We can construct an algorithm $\OptArmIdE(n, \delta)$ ($\delta \leq .1$)  to return the best arm with probability at least $(1 - \delta)$, while the expected sample complexity is $O \left( \sum_{i = 1}^n \left( \frac{\sigma_i^2}{\Delta_i^2} + \frac{1}{\Delta_i} \right) ( \ln \delta^{-1} + \ln \ln \Delta_i^{-1} ) \right)$.
\end{theorem}

For completeness of the paper, we present the proof of Theorem~\ref{THM:EXPECTED-TIME-ALGO} in Appendix~\ref{app:expected-time-algo}.

Note that the square term scales with the variance instead of a constant, which could lead to significant improvement in some cases. We present a specific example that $\OptArmId(n, \delta)$ achieves better performance than other mean-based algorithms as follows.

\begin{example} \label{EXP:1}
Suppose we are given $n$ Bernoulli arms (i.e., the reward of each arm is either $0$ or $1$), the mean reward of arm $i$ is $\theta_i = 1 - \frac{i}{n}$ for $i = 1, 2, \dots, n$. Our variance-dependent algorithm achieves $\delta$-correctness with $O(n \ln n (\ln \delta^{-1} + \ln \ln n))$ samples. In contrast, the expressions in the big-O notations in both \eqref{eq:gap-dependent-bound-1} and \eqref{eq:gap-dependent-bound-2} are $\Omega(n^2 \ln \delta^{-1})$. We show the detailed calculation in Appendix~\ref{app:easy-example}.
\end{example}

Let $[n] = \{1, 2, \dots, n\}$. In the following theorem, we present a lower bound for algorithms aiming to identify the best arm. Therefore, our algorithmic bound \eqref{eq:var-dependent-bound} matches the lower bound up to doubly logarithmic terms.

\begin{theorem} [Restatement of Theorem~\ref{thm:lower-bound-n}]
For any $\sigma_i^2 < 0.1, i \in [n]$  and $0 < \Delta_i < 0.1, i = 2, \dots, n$, there exists an input instance with matching parameters (gaps and variances) such that any $\delta$-correct algorithm ($\delta < 0.1$) needs at least
\begin{align}\label{eq:var-dependent-lower-bound}
\Omega \left( \sum_{i = 1}^n \left( \frac{\sigma_{i}^2}{\Delta_{i}^2} + \frac{1}{\Delta_{i}} \right)  \ln \delta^{-1} \right)
\end{align}
samples.
\end{theorem}

\subsection{Organization and Proof Outline}

In Section~\ref{SEC:VAR-ESTIMATE}, we first describe and analyze a few procedures to estimate the variance of the rewards of a given arm, and the arm's mean reward based on the variance estimation. In Section~\ref{sec:naive-way}, we present a straightforward way to use these procedures to identify the best arm, with the sub-optimal sample complexity $O ( \sum_{i = 1}^n ( \frac{\sigma_i^2}{\Delta_i^2} + \frac{1}{\Delta_i} ) ( \ln \delta^{-1} + \ln \ln \Delta_i^{-1} + \ln n) )$ (note the extra $\ln n$ term comparing with our desired bound \eqref{eq:var-dependent-bound}). Then we develop our main variance-dependent algorithm for best-arm identification in Section~\ref{SEC:EPS-OPTIMAL-ARM} and Section~\ref{SEC:N-ARM}.

In Section~\ref{SEC:EPS-OPTIMAL-ARM}, we present a key technical component, procedure \OptArmEst, to estimate the best-arm's mean reward up to $\epsilon$ precision with probability $(1 - \delta)$ and uses at most $O ( \sum_{i=1}^n ( \frac{\sigma_i^2}{\eps^2} + \frac{1}{\eps} ) ( \ln \delta^{-1} + \ln \ln \eps^{-1}) )$ samples. Note that this bound is similar to that of the median elimination algorithm proposed in \cite{even2002pac} in the sense that both are independent on the reward gap parameters. However, our \OptArmEst~procedure does explore the variance information and forms its strategy accordingly. To achieve this goal, \OptArmEst~uses the idea \emph{grouped median elimination} and iteratively performs the following procedure: first estimate each arm's reward variance and divide the arms into groups, so that arms in the same group have similar reward variance estimations; then perform variance-dependent mean estimation and median elimination within each group. If the variance estimations were always accurate and the arms were all assigned to the desired groups, it would be relatively easy to show that the algorithm makes progress in each iteration (where ``progress'' is defined to be an multiplicative reduction of the total variances of the remaining arms). However, in our analysis, we need substantial technical effort to deal with the mis-placed arms, which is achieved by making very refined upper bounds for the number of mis-placed arms according to the severity of the mistake.

In Section~\ref{SEC:N-ARM}, we use \OptArmEst~as a helper procedure to build our main algorithm. The high level idea here is similar to that of the exponential gap algorithm introduced in \cite{karnin2013almost}. However, due to the non-uniformity of variances among the arms, we have to design a new stopping condition for our iterative algorithm. In Appendix~\ref{sec:LB}, we prove the variance-dependent lower bound result. Finally we conclude the paper by mentioning a few future directions in Section~\ref{sec:future}.

\section{Variance-Dependent Mean Estimation} \label{SEC:VAR-ESTIMATE}
We first build a few subroutines to estimate the variance of the rewards of a given arm (Section~\ref{sec:var-est}), as well as the arm's mean reward based on the variance estimation (Section~\ref{sec:mean-est}). These procedures will be useful in building blocks to design our main algorithm. All missing proofs in this section are deferred to Appendix~\ref{app:proofs-in-sec:var-estimate}.

\subsection{Variance Estimation} \label{sec:var-est}

Our goal of this subsection is to design a procedure to estimate order of the variance of the rewards of a given arm. More specifically, our $\VarEst(i, \delta, \ell)$ (Algorithm~\ref{alg:variance-est}) takes arm $i$, confidence level $\delta$ and a positive number $\ell > 0$ (which is used to control the precision of the estimation) as input, and returns an estimate of the variance $\sigma_i^2$ up to precision $\Theta(2^{-\ell})$. We also need a helper procedure $\VarTest(i, \tau, \delta, c)$ (Algorithm~\ref{alg:variance-test}), which takes arm $i$, threshold $\tau$, confidence parameters $\delta$ and a positive number $c \geq 1$ as input, and checks whether $\sigma_i^2$ is above the threshold $\tau$. 

\begin{algorithm}
\DontPrintSemicolon
\caption{Variance Estimation, $\VarEst(i, \delta, \ell)$ ($\ell > 0$)}
\label{alg:variance-est}
\textbf{Input:} Arm $i$, confidence level $\delta$ and a positive number $\ell$\\
\For {$r \leftarrow 1, 2, 3, \dots $} {
$\tau_r \leftarrow 1 / 2^r$\\
\If {$\tau_r \leq \ell ~\textbf{or}~ \VarTest(i, \tau_r, \delta / e , 80) $} {
\textbf{Output:} $\tau = \tau_r$ as the estimated variance of the rewards of arm $i$
}
}
\end{algorithm}

\begin{algorithm} 
\DontPrintSemicolon
\caption{Variance Test, $\VarTest(i, \tau, \delta, c)$ ($c \geq 1$)}
\label{alg:variance-test}
\textbf{Input:} Arm $i$, threshold $\tau$, confidence level $\delta$ and a positive number $c$\\
 $T \leftarrow \frac{c}{\tau} \ln \delta^{-1} $\\
 Sample arm $i$ for $2T$ times and let $x_1, \dots, x_{2T}$ be the empirical rewards in sequence\\
 $\hat{\sigma}_i^2 \leftarrow \frac{1}{T} \sum_{r = 1}^T \frac{ (x_r - x_{r + T})^2}{2}$\\
\leIf {$\hat{\sigma}_i^2 > \tau $} { \textbf{Output: \true}} { \textbf{\false} } 
\end{algorithm}

The following lemma shows the guarantee for the procedure $\VarTest$. 

\vspace{-1em}
\begin{lemma}~\label{LEMMA:VARIANCE-SAMPLING}
Suppose $\delta \leq e^{-1}$. If $\sigma_i^2  \geq 2\tau$, with probability at least $1 - \delta \cdot \left( \frac{\tau}{\sigma_i^2} \right) ^ c$, $\VarTest(i, \tau, \delta, c)$ outputs \true. If $\sigma_i^2  \leq \tau / 2$, with probability at least $1 - \delta \cdot  \left( \frac{\sigma_i^2}{\tau} \right)^c$, $\VarTest(i, \tau, \delta, c)$ outputs \false. Moreover, the sample complexity is $\frac{2c}{\tau} \ln \delta^{-1}$.
\end{lemma}

Now we present the lemma on the guarantee of the procedure $\VarEst$. Note that Lemma~\ref{LEMMA:VARIANCE-EST} not only shows a lower bound on the success probability of  $\VarEst(i, \delta, \ell)$, but also provides an upper bound on the error probability that depends on the logarithmic distance between the algorithm's output and the real variance $\sigma_i^2$. 

\begin{lemma} \label{LEMMA:VARIANCE-EST}
Suppose $\VarEst(i, \delta, \ell)$ returns $\tau$. Let  $r_m = \lceil | \log_2 \frac{\sigma_i^2 }{\tau} | \rceil$ denote the logarithmic mistake ratio. The algorithm has the following three properties.
\begin{lemlist} 
\item It always holds that $\tau > \ell / 2$ and the sample complexity is $O(\frac{1}{\ell} \ln \delta^{-1})$;  \label{LEMMA:VARIANCE-EST:a}

\item  If $\sigma_i^2 \in ( \ell, 1]$, with probability at least $1 - \delta$, we have $\tau \in [ \sigma_i^2 / 4, 2\sigma_i^2 )$ and the sample complexity is $O\left(  \frac{1}{\sigma_i^2} \ln \delta^{-1} \right)$.  We also have $\Pr[ \tau \geq x ] \leq \delta \cdot 2^{-20 r_m } $ when $x \geq 2\sigma_i^2$ and $\Pr[ \tau \leq x] \leq \delta \cdot 2^{-20 r_m } $ when $x < \sigma_i^2/4$; \label{LEMMA:VARIANCE-EST:b}

\item If $\sigma_i^2 \leq 2\ell$, we have $\Pr[ \tau \geq x ] = O\left( \delta \cdot 2^{-20 r_m }  \right)$ for $x \geq \max\{ 2\ell, 2\sigma_i^2 \} $. \label{LEMMA:VARIANCE-EST:c}
\end{lemlist} 
\end{lemma}

\subsection{Variance-Dependent Mean Estimation} \label{sec:mean-est}

In this section, we present $\MeanEst(i, \eps, \delta)$ (Algorithm~\ref{alg:mean-est}) which estimates the mean reward of a given arm $i$ up to $\eps$ additive error with probability at least $1 - \delta$ with sample complexity depending on $\sigma_i^2$.

At a high level, we first estimate the variance of the rewards of a given arm, then apply Proposition~\ref{prop:bernstein} (Bernstein's Inequality) to control the number of samples needed for an estimate up to the given precision requirement. We show the following lemma.

\begin{algorithm} 
\caption{Mean Estimation, $\MeanEst(i, \eps, \delta)$}
\label{alg:mean-est}
 \textbf{Input:} Arm $i$, accuracy $\eps$, and confidence level $\delta$\\
 $\hat{\sigma}_i^2 \leftarrow \VarEst(i, \delta / 2, \eps)$ \label{alg:mean-est-line2} \\ 
 Sample arm $i$ for $ \left( \frac{8 \hat{\sigma}_i^2}{\eps^2} + \frac{2}{3\eps} \right) \ln \frac{4}{\delta}  $ times and let $\hat{\theta}_i$ denote its empirical mean reward \label{alg:mean-est-line3} \\ 
 \textbf{Output:} $\hat{\theta}_i$ as the estimated mean reward of arm $i$
\end{algorithm}

\begin{lemma} \label{LEMMA:MEAN-EST}
With probability at least  $1 - \delta$,  $\MeanEst(i, \eps, \delta)$ outputs an estimate (namely $\hat{\theta}_i$) of the mean reward of arm $i$ such that $|\hat{\theta}_i - \theta_i | \leq \eps$ and the sample complexity is $O \left( \left( \frac{\sigma_i^2}{\eps^2} + \frac{1}{\eps} \right) \ln \delta^{-1} \right)$.
\end{lemma}

Now we prove a few stronger properties of $\MeanEst$ which will be useful for building our main algorithm. 

\begin{lemma} \label{LEMMA:MEAN-EST-COROL} 
Let $Q$ be the samples used by $\MeanEst(i, \eps, \delta)$. There exists a constant $c > 0$ such that
\begin{lemlist}
\item  $Q \leq \frac{c}{\eps^2} \ln \delta^{-1} $;  \label{LEMMA:MEAN-EST-COROL:a} 
\item for integers $j \geq 3$, we have $\Pr\left[Q \leq c\left( \frac{ j \sigma_i^2 }{\eps^2} + \frac{1}{\eps} \right) \ln \delta^{-1}  \right] \geq 1 - \delta \cdot 2^{-20j}$.  \label{LEMMA:MEAN-EST-COROL:b}
\end{lemlist}
\end{lemma}

\section{Warm-Up: Na\"ive Variance-Dependent Best-Arm Identification} \label{sec:naive-way}

In this section, we present a straightforward way (Algorithm~$\NaiveOptArmId$) of using the variance-dependent procedure $\MeanEst$ to iteratively reject non-optimal arms and finally identify the best arm. The analysis  adopts the union bound on all arms and therefore introduces an extra $\ln |S|$ (where $S$ is the input candidate arms) factor in the sample complexity. In particular, we show the following theorem. The algorithm and missing proofs in this section are deferred to Appendix~\ref{app:mis-material-in-sec:naive-way}.

\begin{theorem} \label{THM:NAIVE-OPT-ARM-ID}
With probability at least $1 - \delta$, the $\NaiveOptArmId(S, \delta)$ algorithm outputs the best arm in $S$ and the sample complexity is $O \left( \sum_{i \in S} \left( \frac{\sigma_i^2}{\Delta_i^2} + \frac{1}{\Delta_i} \right) ( \ln \delta^{-1} + \ln \ln \Delta_i^{-1} + \ln |S| ) \right) $.
\end{theorem}

It is also straightforward to get the following PAC-style statement where an $\eps$-optimal arm denotes an arm whose mean reward is $\eps$-close to that of the best arm in $S$.

\begin{corollary} \label{COL:NAIVE-OPT-ARM-ID}
There exists an algorithm that with probability at least $1 - \delta$, finds an $\eps$-optimal arm in $S$ using at most $O \left( \sum_{i \in S} \left( \frac{\sigma_i^2}{ (\Delta_i^{\eps})^2 } + \frac{1}{ \Delta_i^{\eps} } \right) ( \ln \delta^{-1} + \ln \ln (\Delta_i^{\eps})^{-1} + \ln |S| ) \right) $ samples, where $\Delta_i^{\eps} = \max\{ \Delta_i, \eps \}$. We use $\NaiveOptArmEst(S, \eps, \delta)$ to denote this algorithm.
\end{corollary}

\section{Find an $\eps$-Optimal Arm} \label{SEC:EPS-OPTIMAL-ARM}

Now we start to develop our main algorithm. We use $S_{[i]}$ to denote the index of the $i$-th best arm in $S$. When there is a tie, we break it arbitrarily. In this section, we design a procedure $\OptArmEst(S, \eps, \delta)$ (described in Algorithm~\ref{alg:OptArmEst}) which returns an $\eps$-optimal arm. In particular, we prove the following theorem. All missing proofs in this section are deferred to Appendix~\ref{app:proofs-in-SEC:EPS-OPTIMAL-ARM}.

\begin{theorem} \label{thm:OptArmEst}
With probability at least $1 - \delta$, $\OptArmEst(S, \eps, \delta)$ outputs an arm (denoted by $a$) satisfying $| \theta_a - \theta_{S_{[1]}} | \leq \eps$ and uses $O \left( \sum_{i \in S} \left( \frac{\sigma_i^2}{\eps^2} + \frac{1}{\eps} \right) ( \ln \delta^{-1} + \ln \ln \eps^{-1} )  \right)$ samples.
\end{theorem}

\begin{algorithm}  
\DontPrintSemicolon
\caption{Best Arm Estimation, $\OptArmEst(S, \eps, \delta)$} \label{alg:OptArmEst}
\textbf{Input:} A set of arms $S$, accuracy $\eps$, and confidence level $\delta$\\
 $S_1 \leftarrow \ME(S, \eps / 3, \delta / 3)$ \label{line:OptArmEst-2} \\ 
\leIf { $\eps^{-1} \leq \ln |S| $ } { $S_2 \leftarrow S_1$} { $S_2 \leftarrow \ME(S_1, \eps / 3, \delta / 3)$}  \label{line:OptArmEst-6}
$a \leftarrow \NaiveOptArmEst(S_2, \eps / 3, \delta / 3)$ \label{line:OptArmEst-7} \\ 
\textbf{Output:} Arm $a$ 
\end{algorithm}

$\OptArmEst$ can be viewed as an extension of the Median Elimination algorithm. The number of samples used by neither of them depend on the reward gaps. However, our $\OptArmEst$ algorithm explores the variance information and adapts its strategy accordingly. This procedure is the most technical part of our main algorithm. It employs two subroutines $\ME$ and $\HE$ described in Algorithms~\ref{alg:ME} and \ref{alg:HE}.

\begin{algorithm}  
\caption{Iterative Elimination, $\ME(S, \eps, \delta)$}  \label{alg:ME}
 \textbf{Input:} Arm set $S$, accuracy $\eps$, and confidence level $\delta$\\
 Let $\beta \leftarrow \sqrt{255} / 16 \cdot e^{.001} $, $\eps_r \leftarrow \beta^r(1 - \beta) \eps$, and $\delta_r \leftarrow e^{-.1r} (1 - e^{-.1}) \delta$ for $r \geq 0$ \\
 $T_0 \leftarrow S$, $R_0 \leftarrow \emptyset$, $r \leftarrow 0$ \\
\While {$|T_r| > 10$}
     {
      $\langle T_{r + 1}, R^{r+1} \rangle \leftarrow \HE( T_{r}, \eps_r, \delta_r )$\\
      $R_{r+1} \leftarrow R_r \cup R^{r+1}$\\
      $r \leftarrow r + 1$}
 \textbf{Output:} $T \leftarrow T_{r} \cup R_r$
\end{algorithm}

\begin{algorithm} 
\DontPrintSemicolon
\caption{Grouped Median Elimination, $\HE(S, \eps, \delta)$} \label{alg:HE}
 \textbf{Input:} Arm set $S$, accuracy $\eps$, and confidence level $\delta$ \\
 Let $N \leftarrow \lceil \log_2 (2/\eps) \rceil$ be the number of buckets \\
\lFor {$i \in S$} {
 $\hat{\sigma}_i^2 \leftarrow \VarEst(i, \delta / (2 N^2 ) , \eps) $ \label{alg:HE-line4}  
}
 Define bucket $\hat{B}_j \leftarrow \{ i \in S | 2^{- j } < \hat{\sigma}_i^2 \leq 2^{-j + 1} \}$ for $j = [N] $, and let $T \leftarrow \emptyset$\\
\For {$j \leftarrow 1 \textbf{~to~} N$} {
\eIf {$|\hat{B}_j| \geq 2$} {
 Let $\hat{\theta}_i \leftarrow \MeanEst(i, \eps / 2, \delta / (9N))$ for all $ i \in \hat{B}_j $ \\
 Let $\hat{m}_j$ be the median of the empirical means of the arms in $\hat{B}_j$ \label{alg:HE-line9} \\
 $T_j \leftarrow \hat{B}_j \backslash \{ i \in \hat{B}_j | \hat{\theta}_i < \hat{m}_j  \}$\\
 $T \leftarrow T \cup T_j$
}
{
 Put arm in $ \hat{B}_j $ into the recycle bin $R$
}
}
 \textbf{Output:} $T$ and $R$
\end{algorithm}

Comparing our algorithm with the Median Elimination algorithm in \cite{even2002pac}, we note that the major difference is that we use the \emph{grouped median elimination} ($\HE$) instead. If, in each iteration, we simply eliminate a constant fraction of the arms according to their empirical means, we cannot guarantee that the samples needed in each iteration reduces at an exponential rate and the total work converges, which is the case in Median Elimination. This is because in our algorithm, the sample complexity relates to the total reward variances of the active arms, rather than the number of active arms. This non-uniformity among the arms may admit the scenario where the eliminated arms have small reward variances and the elimination process does not reduce the total variances by a constant fraction after each iteration.

To solve this problem, our $\HE$ procedure partitions the arms into buckets according to their empirical reward variances, so that the arms in the same bucket have similar variances of rewards (up to a multiplicative constant factor). If the partition is perfect (i.e., the empirical estimation matches with the true variances and every arm is assigned to the correct bucket), performing median elimination within each group would successfully reduce the total variances by a constant fraction. 

To deal with variance estimation noise and imperfect partition, we make considerable effort to upper bound the fraction of arms put in wrong buckets, where the bound is very refined and depends on the distance between the desired and empirical buckets. Another consequence of the noise is that, besides the active arm set $T$ returned by $\HE$, we have to introduce a recycle set $R$ of arms. The arms in $R$ do not participate in future rounds of elimination in $\ME$. However, they appear as the returned arms of $\ME$. Indeed, the procedure $\ME$ returns a small set of arms instead of the optimal arm. Finally, we use $\OptArmEst$ to examine this small set again to identify the best arm.

We start the sketch of the analysis of our algorithms by presenting the following statement for $\HE$.

\begin{theorem} \label{thm:HE}
With probability at least $1 - \delta$,  $\HE(S, \eps, \delta)$ outputs two sets $T$ and $R$ of arms and has the following four guarantees:
\begin{thmlist}
\item $|R| = O( \ln \eps^{-1})$; \label{thm:HE:a}
\item  $ \sum_{a \in T} (\sigma_a^2 + \eps) \leq \frac{255}{256} \cdot  \sum_{a \in S} ( \sigma_a^2 + \eps) $; \label{thm:HE:b}
 \item $ | \theta_{(T \cup R)_{[1]} } - \theta_{S_{[1]} } | \leq \eps$; \label{thm:HE:c}
\item uses $O \left( \sum_{i \in S} \left( \frac{\sigma_i^2}{\eps^2} + \frac{1}{\eps} \right) ( \ln \delta^{-1} + \ln \ln \eps^{-1} )  \right)$ samples. \label{thm:HE:d}
\end{thmlist}
\end{theorem}

The proof of Theorem~\ref{thm:HE} is split into three subsections. The first claim is easy to verify and shown in the form of the short Lemma~\ref{LEMMA:HE-CARD-R}. In Section~\ref{sec:ub-E-event}, we define an event $\calE$ (Equation~\eqref{eq:main-alg-main-event}) concerning about the fraction of the arms put in wrong buckets, and use Lemma~\ref{LEMMA:HE-EVENT} to show that $\calE$ holds with high probability $1- \delta/3$. In Section~\ref{sec:mult-reduction-variances}, we prove Lemma~\ref{LEMMA:HE-VAR-CUT}, i.e., $\calE$ implies the second claim of the theorem. In Appendix~\ref{sec:sample-complexity-HE}, we prove Lemmas~\ref{lemma:HE-correctness} and \ref{lemma:HE-sample-comp}, showing that both the probabilities that the third and the fourth claims of the theorem hold are at least $1 - \delta /3$. Finally the theorem is proved by a straightforward union bound.

The following theorem shows the guarantee of $\ME$, and will be proved in Appendix~\ref{sec:proof-thm-ME}.

\begin{theorem} \label{thm:ME}
With probability at least $1 - \delta$, $\ME(S, \eps, \delta)$ outputs an arm set $T$ and has the following three guarantees,
\begin{thmlist}
\item $|T| = O( (\ln |S|)^2 \ln \eps^{-1}) $; \label{thm:ME:a}
\item $| \theta_{ T_{[1]} } - \theta_{ S_{[1]} } | \leq \eps$; \label{thm:ME:b}
\item uses  $O \left( \sum_{i \in S} \left( \frac{\sigma_i^2}{\eps^2} + \frac{1}{\eps} \right) ( \ln \delta^{-1} + \ln \ln \eps^{-1} )  \right)$ samples. \label{thm:ME:c}
\end{thmlist}
\end{theorem}
Finally, with the help of Theorems~\ref{thm:HE} and \ref{thm:ME}, we prove the main theorem on $\OptArmEst$ in Section~\ref{sec:proof-thm-optarmest}.

\subsection{Upper Bounds on Fraction of Arms in Wrong Buckets}\label{sec:ub-E-event}

For notational convenience, for each $\hB_j$ ($j = 1, 2, \dots, N$), we set $l(\hB_j) = 2^{-j}$ and $u(\hB_j) = 2^{-j + 1}$ as the lower and upper bounds on the estimated reward variances of the arms in $\hB_j$. We also introduce the ``ideal'' partition $B_j = \{ i \in S \mid 2^{- j } < \sigma_i^2 \leq 2^{-j + 1} \}$ for $j = 1, 2, \dots, N - 1$ and $B_N = \{ i \in S \mid 0 \leq \sigma_i^2 \leq 2^{- N + 1} \}$. Similarly, we set $l(B_j) = 2^{-j}$ for $j = 1, 2, \dots N - 1$ and $u(B_j) = 2^{-j + 1}$ for $j = 1, 2, \dots, N$, with the exception that $l(B_N) = 0$.

Now we list the following simple facts about the procedure $\HE$.
\begin{lemma} \label{LEMMA:PARTITION-OF-S}
$\{\hat{B}_1, \hat{B}_2, \dots, \hat{B}_N\}$ is a partition of $S$.
\end{lemma}

\begin{lemma} \label{LEMMA:HE-CARD-CUT}
If $| \hat{B}_j | \geq 2$, there is $|T_j| \leq \frac{2}{3} | \hat{B}_j |$.
\end{lemma}

\begin{lemma} \label{LEMMA:HE-CARD-R}
$|R| = O( \ln \eps^{-1})$. 
\end{lemma}

We define $\mathcal{E}$ to be the event
\begin{multline} \label{eq:main-alg-main-event}
 \Big\{ | B_i \cap \hat{B}_j | \\ < |B_i| \cdot 2^{- 10|i - j|} \cdot N^{-1} \textrm{~for~} \forall |i - j | \geq 3 \Big\}.
\end{multline}
In words, it means that the fraction of the arms that are empirically put in a wrong bucket becomes exponentially small as the error distance increases. We now show such an event happens with high probability, which is the main statement of this subsection.

\begin{lemma} \label{LEMMA:HE-EVENT}
$\Pr[ \mathcal{E} ] \geq 1 - \delta / 3$.
\end{lemma}

\subsection{Procedure $\HE$: Multiplicative Reduction of the Total Variances}\label{sec:mult-reduction-variances}

We say that $B_i$ \textit{pollutes} $\hB_j$  (or  $\hB_j$ is \textit{polluted} by $B_i$) if and only if $| B_i \cap \hat{B}_j | > | \hat{B}_j | \cdot 2^{-5|i - j|}$. Intuitively, this means that too many arms (those are supposed to be in $B_i$) are incorrectly put in $\hB_j$. Note that the definition of ``too many'' is in terms of the fraction compared to $|\hB_j|$  rather than $|B_i|$ as defined in the event $\calE$. If $\hat{B}_j$ is polluted by some $B_i$ where $|i - j| \geq 3$, we say that $\hat{B}_j$ is \emph{bad}. Otherwise, we say that $\hat{B}_j$ is \emph{good}.

The following lemma shows that for a good bucket $\hB_j$, as long as it is not the last three buckets, the arms discarded from the bucket aggregate a constant fraction of variances.

\begin{lemma} \label{LEMMA:HE-SITUA1}
Given that $j \leq N - 3$, if $|\hat{B}_j| \geq 2$ and $\hat{B}_j$ is good, there is $\sum_{a \in T_j} \sigma_a^2 \leq \frac{127}{128} \cdot \sum_{a \in \hat{B}_j} \sigma_a^2$.
\end{lemma}

\begin{corollary} \label{COROL:HE-SITUA1}
Given that $j \leq N - 3$, if $|\hat{B}_j| \geq 2$ and $\hat{B}_j$ is good, there is $\sum_{a \in T_j} ( \sigma_a^2 + \eps ) \leq \frac{127}{128} \cdot \sum_{a \in \hat{B}_j} ( \sigma_a^2 + \eps )$.
\end{corollary}

We now prove a similar statement as Corollary~\ref{COROL:HE-SITUA1}, but for the last three buckets.

\begin{lemma}  \label{LEMMA:HE-SITUA2}
Given that $j \geq N - 2$, if $|\hat{B}_j| \geq 2$ and $\hat{B}_j$ is good, there is $\sum_{a \in T_j} ( \sigma_a^2 + \eps ) \leq \frac{127}{128} \cdot \sum_{a \in \hat{B}_j} ( \sigma_a^2 + \eps ) $.
\end{lemma}

The following two lemmas control the total reward variances of the arms in a polluted bucket.

\begin{lemma} \label{LEMMA:HE_SITUA3}
Conditioning on $\mathcal{E}$, if $ \hat{B}_j $ is polluted by some $B_i$ where $i \leq N - 1$ and $|i - j| \geq 3$, we have $\sum_{ a \in \hat{B}_j } \sigma_a^2 \leq N^{-1} \cdot \sum_{ a \in S } \sigma_a^2 \cdot \frac{1}{256}$.
\end{lemma}

\begin{lemma} \label{LEMMA:HE_SITUA4}
Conditioning on $\mathcal{E}$, if $ \hat{B}_j $ is only polluted by $B_N$ where $|N - j| \geq 3$, we have $\sum_{ a \in \hat{B}_j } \sigma_a^2 \leq N^{-1} \cdot \sum_{ a \in S }  (\sigma_a^2 + \eps)  \cdot \frac{1}{1024}$.
\end{lemma}

Now, we show that with high probability the total reward variances of the active arms reduce by a constant fraction after the procedure $\HE$. In particular, we prove the following lemma.

\begin{lemma} \label{LEMMA:HE-VAR-CUT}
Conditioning on event $\mathcal{E}$, we have $ \sum_{a \in T} (\sigma_a^2 + \eps) \leq \frac{255}{256} \sum_{a \in S} ( \sigma_a^2 + \eps) $.
\end{lemma}
\begin{proof}
According to Lemma~\ref{LEMMA:HE_SITUA3} and Lemma~\ref{LEMMA:HE_SITUA4}, if $\hat{B}_j$ is polluted by some $B_i$ where $|i -j | \geq 3$, there is $\sum_{ a \in \hat{B}_j } \sigma_a^2 \leq N^{-1} \cdot \sum_{ a \in S }  (\sigma_a^2 + \eps)  \cdot \frac{1}{256}$ which implies $\sum_{j, \hat{B}_j \textrm{~is~good~} }\sum_{ a \in \hat{B}_j } \sigma_a^2 \geq \frac{255}{256} \sum_{ a \in S }  (\sigma_a^2 + \eps) $. Hence there is
\begin{align} 
& \frac{ \sum_{a \in S} ( \sigma_a^2 + \eps) - \sum_{a \in T} (\sigma_a^2 + \eps) }{ \sum_{a \in S} ( \sigma_a^2 + \eps)  } \notag \\ 
\geq {} & \frac{ \sum_{j, \hat{B}_j \textrm{~is~good~} }\sum_{ a \in \hat{B}_j \backslash T_j } \sigma_a^2 }{ \frac{256}{255} \cdot \sum_{j, \hat{B}_j \textrm{~is~good~} }\sum_{ a \in \hat{B}_j } \sigma_a^2 } \notag \\
\geq {} & \frac{255}{256} \cdot \min_{j, \hat{B}_j \textrm{~is~good~} }  \frac{ \sum_{ a \in \hat{B}_j \backslash T_j } \sigma_a^2 }{ \sum_{ a \in \hat{B}_j } \sigma_a^2 } . \label{equ:HE-var-cut-1}
\end{align}

When $|\hat{B}_j| = 1$, $T_j = \emptyset$ which implies $ \frac{ \sum_{ a \in (\hat{B}_j - T_j) } \sigma_a^2 }{ \sum_{ a \in \hat{B}_j } \sigma_a^2 } = 1$. When $\hat{B}_j$ is good and $|\hat{B}_j| \geq 2$, according to Corollary~\ref{COROL:HE-SITUA1} and Lemma~\ref{LEMMA:HE-SITUA2}, there is $ \frac{ \sum_{ a \in (\hat{B}_j - T_j) } \sigma_a^2 }{ \sum_{ a \in \hat{B}_j } \sigma_a^2 } \geq \frac{1}{128}$. Therefore, we have 
$
\eqref{equ:HE-var-cut-1} \geq \frac{255}{256} \cdot \frac{1}{128} \geq \frac{1}{256},
$
which concludes the proof of this lemma.
\end{proof}

\subsection{Analysis of the $\OptArmEst$ algorithm}\label{sec:proof-thm-optarmest}

Now we are ready to analyze the $\OptArmEst$ algorithm and prove the main theorem (Theorem~\ref{thm:OptArmEst}) of this subsection.

First, we define the following three events about the $\OptArmEst$ procedure. Let $c$ be the hidden constant in Corollary~\ref{COL:NAIVE-OPT-ARM-ID} and Theorem~\ref{thm:ME}.

\begin{itemize}
\item Let $\calE_1$ denote the event $|S_1| \leq c (\ln |S|)^2 \ln \eps^{-1} $, $| \theta_{ (S_1)_{[1]} } - \theta_{S_{[1]} } | \leq \eps / 3$, and the sample complexity of Line~2 is at most
$
c \sum_{i \in S} \left( \frac{\sigma_i^2}{\eps^2} + \frac{1}{\eps} \right) ( \ln \delta^{-1} + \ln \ln \eps^{-1} ).
$
\item Let $\calE_2$ denote the event $|S_2| = c (\ln |S_1|)^2 \ln \eps^{-1} $, $| \theta_{ (S_2)_{[1]} } - \theta_{ (S_1)_{[1]} } | \leq \eps / 3$, and the sample complexity of Line~3 is at most
$
c \sum_{i \in S_1} \left( \frac{\sigma_i^2}{\eps^2} + \frac{1}{\eps} \right) ( \ln \delta^{-1} + \ln \ln \eps^{-1} ).
$
\item Let $\calE_3$ denote the event $| \theta_{a} - \theta_{ (S_2)_{[1]} } | \leq \eps / 3$ and the sample complexity of Line~4 is at most
 $
 c \sum_{i \in S_2} \left( \frac{\sigma_i^2}{\eps^2} + \frac{1}{\eps} \right) $ $ ( \ln \delta^{-1} + \ln \ln \eps^{-1} + \ln |S_2|).
 $
\end{itemize}

\begin{myproof}{Proof of Theorem~\ref{thm:OptArmEst}}
By Theorem~\ref{thm:ME}, we have $\Pr[ \calE_1 ] \geq 1 - \delta / 3$ and $\Pr[ \calE_2 ] \geq 1 - \delta / 3$. By Corollary \ref{COL:NAIVE-OPT-ARM-ID}, we have $\Pr[ \calE_3 ] \geq 1 - \delta / 3$. Conditioning on event $\calE_1 \wedge \calE_2 \wedge \calE_3$ which happens with probability $1 - \delta$, we will show both claims of Theorem~\ref{thm:OptArmEst} hold.

The first claim is because of $| \theta_a - \theta_{S_{[1]} } | \leq | \theta_a - \theta_{ (S_2)_{[1]} } | + | \theta_{ (S_2)_{[1]} } - \theta_{ (S_1)_{[1]} } | + | \theta_{ (S_1)_{[1]} } - \theta_{ S_{[1]} } | \leq \eps$.

Now we focus on the second claim (about the sample complexity). It suffices to show that the sample complexity of Line~4 meets the desired asymptotic upper bound. We discuss the following two cases.

\paragraph{Case 1:  $\eps^{-1} \leq \ln |S|$.} Note that $S_2 = S_1$ and 
$
O \left( \sum_{i \in S_2} \left( \frac{\sigma_i^2}{\eps^2} + \frac{1}{\eps} \right) \ln |S_2|  \right) = O\left( \frac{|S_1| \ln |S_1|}{\eps^2} \right) = O\left(  \frac{\ln |S| |S_1| \ln |S_1|}{\eps} \right) = O\left( \frac{|S|}{\eps} \right),
$
where the last equality is due to $|S_1| = O( (\ln |S|)^2 \ln \eps^{-1})$. Hence, the sample complexity of Line~\ref{line:OptArmEst-7} is 
\begin{align*}
 & O \left( \sum_{i \in S_2} \left( \frac{\sigma_i^2}{\eps^2} + \frac{1}{\eps} \right) ( \ln \delta^{-1} + \ln \ln \eps^{-1} + \ln |S_2|)  \right)  \\
= {} & O \Bigg( \sum_{i \in S_2} \left( \frac{\sigma_i^2}{\eps^2} + \frac{1}{\eps} \right) ( \ln \delta^{-1} \\ 
& + \ln \ln \eps^{-1})  \Bigg) + O \left( \sum_{i \in S_2} \left( \frac{\sigma_i^2}{\eps^2} + \frac{1}{\eps} \right) \ln |S_2|  \right) \\
= {} & O \Big( \sum_{i \in S} \left( \frac{\sigma_i^2}{\eps^2} + \frac{1}{\eps} \right) ( \ln \delta^{-1} + \ln \ln \eps^{-1})  \Big) + O\left( \frac{|S|}{\eps} \right) \\ 
= {} & O \left( \sum_{i \in S} \left( \frac{\sigma_i^2}{\eps^2} + \frac{1}{\eps} \right) ( \ln \delta^{-1} + \ln \ln \eps^{-1})  \right).
\end{align*}

\paragraph{Case 2:  $\eps^{-1} > \ln |S|$.}  Note that 
$
\ln |S_2| =  O( \ln \ln |S_1| + \ln \ln \eps^{-1} ) = O( \ln \ln \ln |S| + \ln \ln \eps^{-1} ) = O( \ln \ln \eps^{-1} ),
$
where the first and second equalities are due to $|S_2| = O( (\ln |S_1|)^2 \ln \eps^{-1} )$ and $|S_1| = O( (\ln |S|)^2 \ln \eps^{-1} )$ respectively. Hence, the sample complexity of Line~\ref{line:OptArmEst-7} is 
$
O \Big( \sum_{i \in S_2} \left( \frac{\sigma_i^2}{\eps^2} + \frac{1}{\eps} \right) ( \ln \delta^{-1} + \ln \ln \eps^{-1} + \ln |S_2|)  \Big)  
= O \left( \sum_{i \in S} \left( \frac{\sigma_i^2}{\eps^2} + \frac{1}{\eps} \right) ( \ln \delta^{-1} + \ln \ln \eps^{-1})  \right).  
$

In both cases, the sample complexity of Line~\ref{line:OptArmEst-7} is $O \left( \sum_{i \in S} \left( \frac{\sigma_i^2}{\eps^2} + \frac{1}{\eps} \right) ( \ln \delta^{-1} + \ln \ln \eps^{-1})  \right)$. Therefore, the sample complexity of the whole procedure also meets the desired upper bound.
\end{myproof}

\section{The Main Variance-Dependent Algorithm} \label{SEC:N-ARM}
Now we are ready to present the main variance-dependent best arm identification algorithm $\OptArmId(n, \delta)$ with the help of $\MeanEst$ and $\OptArmEst$ developed in previous sections. All missing proofs in this section are deferred to Appendix~\ref{app:proofs-in-sec:n-arm}.

\begin{theorem} \label{thm:opt-arm-id}
With probability at least $1 - \delta$, $\OptArmId(n, \delta)$ outputs the best arm and the number of samples used is $O \left( \sum_{i=1}^n \left( \frac{\sigma_i^2}{\Delta_i^2} + \frac{1}{\Delta_i} \right) ( \ln \delta^{-1} + \ln \ln \Delta_i^{-1} ) \right) $.
\end{theorem}

\begin{algorithm} 
\DontPrintSemicolon
\caption{Variance-Dependent Best Arm Identification,  $\OptArmId(n, \delta)$} \label{alg:opt-arm-id}
\textbf{Input:} Arm set $S=[n]$ and confidence level $\delta$ \\
 $S_1 \leftarrow S$, $r \leftarrow 1$ \\
\While {$|S_r| > 1$} {
 Set $\eps_r \leftarrow 1 / 2^{r+2} $ and $\delta_r \leftarrow 1 / (2r^2) \cdot \delta$ \\
\lFor {$i \in S_r$} {
$\hat{\theta}_i^r \!\leftarrow \MeanEst(i, \frac{\eps_r}{2}, \frac{\delta_r}{18} )$ \label{line-7:opt-arm-id}
}
$a_r \leftarrow \OptArmEst( S_r, \frac{\eps_r}{2}, \frac{\delta_r}{18} ) $ \label{line-8:opt-arm-id} \\ 
$a_r^* \leftarrow \OptArmEst( S_r \backslash \{ a_r \}, \frac{\eps_r}{2}, \frac{\delta_r}{18} ) $ \label{line-9:opt-arm-id} \\
\lIf {$| \hat{\theta}_{a_r}^r - \hat{\theta}_{a_r^*}^r | > 2 \eps_r$} {  
\textbf{Output:} $a_r$ 
} \label{line-10:opt-arm-id}
$S_{r + 1} \leftarrow S_r \backslash \{ i \in S_r | \hat{\theta}_i^r < \hat{\theta}_{a_r}^r - \eps_r \}$ \\
$r \leftarrow r + 1$
}
\textbf{Output:} The remaining arm in $S_r$
\end{algorithm}

We present the details of $\OptArmId(n, \delta)$ in Algorithm~\ref{alg:opt-arm-id}. It has a similar structure to that of the Exponential Gap Elimination algorithm in \cite{karnin2013almost} as our algorithm also keeps a confidence interval $\eps_r$ which halves after each round. Within a round, we estimate the mean reward of each arm up to confidence interval $\eps_r$ and an arm will be discarded if its estimation is $\eps_r$ below that of the best arm. However, due to non-uniformity of the reward variances of the arms, we cannot repeat this process until there is only one arm left (as is done in the Exponential Gap Elimination algorithm), otherwise the sample complexity would not satisfy the desired upper bound. Instead, we design a new stopping condition (Line~8) which may be triggered earlier.

The proof of Theorem~\ref{thm:opt-arm-id} is split into two parts: correctness (the best arm is identified with high probability proved by Lemma~\ref{lemma:opt-arm-id-correct} in Section~\ref{sec:main-correct}) and sample complexity (proved by Lemma~\ref{lemma:opt-arm-id-samp} in Section~\ref{sec:main-sample-complexity}). We finally obtain Theorem~\ref{thm:opt-arm-id} by combining these two lemmas with a union bound.

The rest of this section is devoted to the proof of Theorem~\ref{thm:opt-arm-id}.

\subsection{Correctness} \label{sec:main-correct}

We use $\calM_1$ to denote the event $ \hat{\theta}_{ S_{[1]} }^r \geq \hat{\theta}_{a_r}^r - \eps_r$ for every round  $r$, and use $\calM_2 $ to denote the event that $\OptArmId(n, \delta)$ terminates with $r = O( \ln \Delta_2^{-1} )$ and returns the best arm. We have the following two lemmas. 

\begin{lemma} \label{LEMMA:BEST-ARM-REMAINS}
$\Pr[ \calM_1 ] \geq 1 - \delta / 9$.
\end{lemma}

\begin{lemma} \label{LEMMA:BEST-ARM-OUTPUT}
$
\displaystyle{\Pr\left[\calM_2  |\calM_1 \right]} \geq 1 - 2\delta / 9 $.
\end{lemma}

We now show the correctness lemma as follows.

\begin{lemma} \label{lemma:opt-arm-id-correct}
With probability at least $1 - \delta / 3$,  $\OptArmId(n, \delta)$ terminates with $r = O( \ln \Delta_2^{-1} )$ and returns the best arm.
\end{lemma}

\begin{proof}
It suffices to prove $\Pr[\calM_2] \geq 1 - \delta/3$. By Lemma~\ref{LEMMA:BEST-ARM-REMAINS} and \ref{LEMMA:BEST-ARM-OUTPUT}, we have $\Pr[ \calM_2 ] \geq \Pr[ \calM_2 | \calM_1 ] \Pr[ \calM_1 ] \geq 1- \delta / 3$.
\end{proof}

\subsection{Sample Complexity}\label{sec:main-sample-complexity}

For each $1  \leq s \leq \lceil \log_2(1/\Delta) + 1 \rceil$, we define the set 
$
A_s = \{ i \in S \mid 2^{-s} < \Delta_i \leq 2^{-s+1} \},
$
and let $n_s = |A_s|$. Also, we denote the set of arms from $A_s$ surviving after round $r$ by $S_{r, s} = S_r \cap A_s$.

We will show that from round $s$ onwards, every sub-optimal arm in $A_s$ is eliminated with high probability. Specifically, we show the following lemma. 

\begin{lemma} \label{LEMMA:ELIM-PROB}
Conditioning on $\calM_1$, with probability at least $1 - \delta_r / 4$, we have $\hat{\theta}_i^r < \hat{\theta}_{a_r}^r - \eps_r$ for any arm $i \in S_{r-1,s}$ and round $ r \geq s$.
\end{lemma}

Let $I_{i}^{r}$ denote the random variable $\mathds{1} \{ i \in S_r \}$. We also define  
$
T_i^r = \left( \frac{\sigma_i^2}{\eps_r^2} + \frac{1}{\eps_r} \right) ( \ln \delta_r^{-1} + \ln \ln \eps_r^{-1} ) .
$

In the desired event (which is explicitly defined by event $\calM_3$ and analyzed in Lemma~\ref{LEMMA:OPT-ARM-ID-ROUND-SAMP} soon afterwards), we may bound the number of pulls to arm $i$ in round $r$ by $I_i^r T_i^r$. In light of this, the following two lemmas help to upper-bound the number of pulls to the sub-optimal arms where $c$ is a constant.

\begin{lemma} \label{LEMMA:SAMP-PER-ARM}
Conditioning on $\calM_1$, we have that with probability at least $1 - \left(\frac{\delta}{8}\right)^j$, $\sum_{r = 1}^{+\infty} I_i^r T_i^r \leq c  4^j \left( \frac{\sigma_i^2}{\Delta_i^2} + \frac{1}{\Delta_i} \right) ( \ln \delta^{-1} + \ln \ln \Delta_i^{-1} )$ for $i \neq S_{[1]} $.
\end{lemma}

\begin{lemma} \label{LEMMA:SAMP-BY-SUB-OPTIMAL-ARMS}
Conditioning on $\calM_1$, we have that with probability at least $1 - \frac{\delta}{18}$, $\sum_{i \neq S_{[1]} } \sum_{r = 1}^{+\infty} I_i^r T_i^r \leq
O \left( \sum_{i \neq S_{[1]} } \left( \frac{\sigma_i^2}{\Delta_i^2} + \frac{1}{\Delta_i} \right) ( \ln \delta^{-1} + \ln \ln \Delta_i^{-1} ) \right) $.
\end{lemma}

The following lemma helps to upper-bound the number of the pulls to the best arm.

\begin{lemma} \label{LEMMA:SAMP-BY-OPTIMAL-ARMS}
When $\calM_2$ happens, we have 
$
 \sum_{r = 1}^{+\infty} I_{S_{[1]} }^r T_{S_{[1]} }^r = O \left( \left( \frac{\sigma_1^2}{\Delta_1^2} + \frac{1}{\Delta_1} \right) ( \ln \delta^{-1} + \ln \ln \Delta_1^{-1} )  \right).
 $
\end{lemma}

We use $\calM_3$ to denote the event that, for each $r$, the  number of samples used in round $r$ is $\sum_{i=1}^n O( I_i^r T_i^r )$. The following lemma shows that $\calM_3$ happens with high probability.

\begin{lemma} \label{LEMMA:OPT-ARM-ID-ROUND-SAMP}
$\Pr[ \calM_3 ] \geq 1 - \delta / 6$ .
\end{lemma}

We are now ready to prove the following lemma on the sample complexity of $\OptArmId$.

\begin{lemma} \label{lemma:opt-arm-id-samp}
With probability at least $1 - 2 \delta / 3$,  the sample complexity of  $\OptArmId(n, \delta)$ is 
\begin{align*}
O \left( \sum_{i = 1}^n \left( \frac{\sigma_i^2}{\Delta_i^2} + \frac{1}{\Delta_i} \right) ( \ln \delta^{-1} + \ln \ln \Delta_i^{-1} ) \right).
\end{align*}
\end{lemma}

\begin{proof}
Note that $\Pr[ \calM_1 ] \geq 1 - \delta / 9$ by Lemma~\ref{LEMMA:BEST-ARM-REMAINS}. Further by Lemma~\ref{LEMMA:SAMP-BY-SUB-OPTIMAL-ARMS}, with probability at least $(1 - \delta / 9)(1 - \delta / 18) \geq 1 - \delta / 6$, we have
$\sum_{i \neq S_{[1]} } \sum_{r = 1}^{+\infty} I_i^r T_i^r = 
O \left( \sum_{i \neq S_{[1]} } \left( \frac{\sigma_i^2}{\Delta_i^2} + \frac{1}{\Delta_i} \right) ( \ln \delta^{-1} + \ln \ln \Delta_i^{-1} )  \right) $.
Note that $\Pr[ \calM_2 ] \geq 1 - \delta / 3$ by Lemma~\ref{lemma:opt-arm-id-correct}. Further by Lemma~\ref{LEMMA:SAMP-BY-OPTIMAL-ARMS}, with probability at least $1 - \delta / 3$, it holds that
$\sum_{r = 1}^{+\infty} I_{ S_{[1]} }^r T_{ S_{[1]} }^r = 
O \left( \left( \frac{\sigma_1^2}{\Delta_1^2} + \frac{1}{\Delta_1} \right) ( \ln \delta^{-1} + \ln \ln \Delta_1^{-1} )  \right) $.
Via a union bound, with probability at least $1 - \delta / 2$, 
\begin{multline} \label{equ-3:opt-arm-id-samp}
\sum_{i=1}^n \sum_{r = 1}^{+\infty} I_i^r T_i^r \\ = 
O \left( \sum_{i=1}^n \left( \frac{\sigma_i^2}{\Delta_i^2} + \frac{1}{\Delta_i} \right) ( \ln \delta^{-1} + \ln \ln \Delta_i^{-1} )  \right).
\end{multline}

Note that $\Pr[ \calM_3 ] \geq \delta / 6$ by Lemma~\ref{LEMMA:OPT-ARM-ID-ROUND-SAMP}. Conditioning on \eqref{equ-3:opt-arm-id-samp} and event $\calM_3$ which happens  with probability at least $1 - 2 \delta / 3$ (via a union bound), the sample complexity of algorithm $\OptArmId(n, \delta)$ is 
$
 \sum_{r = 1}^{+\infty} \sum_{i = 1}^n O( I_i^r T_i^r ) = O\left( \sum_{i=1}^n \sum_{r = 1}^{+\infty} I_i^r T_i^r \right) = 
O \left( \sum_{i=1}^n \left( \frac{\sigma_i^2}{\Delta_i^2} + \frac{1}{\Delta_i} \right) ( \ln \delta^{-1} + \ln \ln \Delta_i^{-1} )  \right).
$
\end{proof}

\section{Conclusion and Future Works}\label{sec:future}
In this paper, we present a variance-dependent best arm identification algorithm and the nearly matching sample complexity lower bound.

While our algorithm almost achieves theoretical optimality, its empirical performance suffers from the large constant factors introduced by multiple subroutines. It is worthwhile to design algorithms with better empirical performance and the same sample complexity bound. The UCB-style algorithms (e.g.\  lil'UCB in \cite{jamieson2014lil}) are a very promising direction towards this end.

On the theoretical side, we believe that it is promising to combine our approach with the ideas in \cite{chen2017towards} and improve the doubly-logarithmic terms in our sample complexity bound. It is very interesting to investigate the ultimate sample complexity of the problem.

\begin{acknowledgements}
We want to thank Yuan Zhou for providing valuable ideas and many helpful discussions. Pinyan Lu is supported by Science and Technology Innovation 2030 –“New Generation of Artificial Intelligence” Major Project No.(2018AAA0100903), NSFC grant 61922052 and 61932002, Innovation Program of Shanghai Municipal Education Commission, Program for Innovative Research Team of Shanghai University of Finance and Economics, and the Fundamental Research Funds for the Central Universities. Chao Tao is supported in part by NSF IIS-1633215, NSF CCF-1844234, and NSF  CCF-2006591.
\end{acknowledgements}

\onecolumn
\appendix

\title{Variance-Dependent Best Arm Identification (Supplementary Material)}



\setcounter{theorem}{0}
\renewcommand{\thetheorem}{\Alph{section}.\arabic{theorem}}

\section{Concentration Inequalities}

\begin{proposition}[Multiplicative Chernoff Bound]
\label{prop:chernoffbound}
Let $X_i$ ($1 \leq i \leq n$) be i.i.d. random variables supported on $[0, 1]$. Let $X = \frac{1}{n} \sum_{i = 1}^n X_i$ and $\E X_1 = \mu$. We have that
    \[ 
    \Pr[ X < (1 - \eps) \mu ] <  \left( \frac{ e^{ - \eps} }{(1 - \eps)^{ (1 - \eps) }} \right)^{n\mu},  \forall \eps \in (0, 1), \text{~and} 
    \]
    \[ 
    \Pr[ X > (1 + \eps) \mu ] <  \left( \frac{ e^{\eps} }{(1 + \eps)^{ (1+\eps) }} \right)^{n\mu}, \forall \eps > 0. 
    \]
\end{proposition}

\begin{proposition}[Bernstein Inequality] \label{prop:bernstein}
Let $X_i$ ($1 \leq i \leq n$) be \textit{i.i.d.}\ random variables. Suppose $|X_i| \leq M$ holds almost surely, for any $i$. Let $X = \frac{1}{n} \sum_{i = 1}^n X_i$, $\E [X_1] = \mu$ and $ \Var X_1 = \sigma^2 $. Then, for all positive $t$, it holds that

\[
\Pr[ | X - \mu | > t ] \leq 2\exp \left( - \frac{n t^2 / 2}{ \sigma^2 + \frac{1}{3} M t } \right).
\]

\end{proposition}

\section{Proof of Theorem~\ref{THM:EXPECTED-TIME-ALGO}}
\label{app:expected-time-algo}

Let $\mathbb{A}_i$ denote the algorithm $\OptArmId(n, \delta / 2^i)$. Algorithm~$\OptArmIdE(n, \delta)$ is constructed as follows. It is easy to verify that after the finish of round $r$, $\mathbb{A}_i$ makes $\lfloor r / 2^i \rfloor$ samples. Therefore, after round $r$, the total number of samples made is at most $\sum_{i = 1}^{\lfloor \log_2 r \rfloor} \lfloor r/2^i \rfloor \leq r$.

\begin{algorithm} 
\DontPrintSemicolon
\caption{$\OptArmIdE(n, \delta)$}
\label{alg:expected-time}
\textbf{Input:} Number of arms $n$ and confidence level $\delta$ \\
\For {$r \leftarrow 1 \textbf{~to~} +\infty$} {
\For {$i \leftarrow 1 \textbf{~to~} \lfloor \log_2 r \rfloor$} {
\If {$2^i | r$} {
 Run $\mathbb{A}_i$ until one of the following two conditions is satisfied: \\
 \Indp
 i) $\mathbb{A}_i$ needs to sample some arm, or \\
 ii) $\mathbb{A}_i$ terminates with an output arm $a$ \\
 \Indm
\lIf {i) is satisfied} {
 Sample the arm for one time and feed the observation to $\mathbb{A}_i$
}
\lElse{
 \textbf{Output:} Arm $a$
}
}
}
}
\end{algorithm}

Before proceeding, let us define some symbols. For simplicity, we define $\Phi \coloneqq \sum_{i = 1}^n \left( \frac{\sigma_i^2}{\Delta_i^2} + \frac{1}{\Delta_i} \right)$ and $\Psi \coloneqq \sum_{i = 1}^n \left( \frac{\sigma_i^2}{\Delta_i^2} + \frac{1}{\Delta_i} \right) \ln \ln \Delta_i^{-1}$. Let $c$ be a constant hidden in the big-O notation of Theorem~\ref{THM:EXPECTED-TIME-ALGO} and $\mathcal{G}_i$ denote the event that $\mathbb{A}_i$ outputs the best arm and the sample complexity is $c( \Phi \ln (2^i / \delta) + \Psi)$. 

We first prove the $\delta$-correctness of Algorithm~$\OptArmIdE(n, \delta)$. By Theorem~\ref{THM:EXPECTED-TIME-ALGO}, we have $\Pr[ \mathcal{G}_i ] \geq 1- \delta / 2^i$. Let $\mathcal{G}$ denote the event $\bigwedge_{i = 1}^{+\infty} \mathcal{G}_i$. Via a union bound, we have $\Pr[ \mathcal{G}] \geq 1 - \sum_{i = 1}^{+\infty} \Pr[ \mathcal{G}_i ] = 1 - \delta \sum_{i = 1}^{+\infty} 2^{-i} \geq 1 - \delta$. We now condition on the event $\mathcal{G}$ until the end of this paragraph. Note that during the first $r$ rounds of  $\OptArmIdE$, $\mathbb{A}_1$ makes at least $\lfloor r / 2 \rfloor$ samples. Hence, $\OptArmIdE$ must stop with $r \leq 2c(\Phi \ln (2 / \delta) + \Psi) + 2$.  Since every $\mathbb{A}_i$ outputs the best arm, so does $\OptArmIdE$. Therefore, the first part of  Theorem~\ref{THM:EXPECTED-TIME-ALGO} is proved. 

Next, we focus on proving the upper bound of the expected sample complexity of $\OptArmIdE$. Let $\mathcal{H}_i$ denote the event $\mathcal{G}_1^c \wedge \mathcal{G}_2^c \wedge \cdots \wedge \mathcal{G}_{i - 1}^c \wedge \mathcal{G}_i$. Since samples are independently made, we have 
\begin{equation} \label{app:expected-time:eq:1}
\Pr[ \mathcal{H}_i ] = \Pr[ \mathcal{G}_1^c] \cdots \Pr[ \mathcal{G}_{i - 1}^c ] \Pr[ \mathcal{G}_i] \leq  \prod_{j = 1}^{i - 1} (\delta / 2^j) \leq \delta^{i - 1}.
\end{equation}

We claim that the set $\{ \mathcal{H}_i \}_{i = 1}^{ +\infty }$ is almost a partition of the whole probability space $\Omega$ i.e., it satisfies the following two properties:
\begin{enumerate}[label=\roman*)]
\item $\mathcal{H}_i \cap \mathcal{H}_j = \emptyset$, for any $i < j$, and
\item $\Pr\left[ \bigcup_{i = 1}^{+\infty} \mathcal{H}_i \right] = 1$.
\end{enumerate}
The first property can be easily verified since $\mathcal{H}_i \subset \mathcal{G}_i$ and $\mathcal{H}_j \subset \mathcal{G}_i^c$. For the second property, since the partial sum $\bigcup_{i = 1}^j \mathcal{H}_i $ is equal to $\Omega \backslash ( \mathcal{G}_1^c \wedge \mathcal{G}_2^c \wedge \cdots \wedge \mathcal{G}_j^c )$. Hence $\Pr\left[ \bigcup_{i = 1}^{+\infty} \mathcal{H}_i \right] = \lim_{i \to +\infty}(1 - \Pr[ \mathcal{G}_1^c \wedge \mathcal{G}_2^c \wedge \cdots \wedge \mathcal{G}_i^c ]) \geq 1 - \lim_{i \to +\infty} \prod_{j =1}^{i} \delta / 2^j = 1$. Therefore this claim is proved.

Let $T$ and $T_i$ be the random variables representing the sample complexities of $\OptArmIdE$ and $\mathbb{A}_i$ respectively. Note that during the first $r$ rounds of $\OptArmIdE$, $\mathbb{A}_i$ makes $\lfloor r / 2^i \rfloor $ samples. Hence we have $T \leq 2^i (T_i + 1)$. Further by the law of total expectation, there is 
\begin{align*}
\E [ T ] & = \sum_{i = 1}^{+ \infty} \Pr[\mathcal{H}_i] \E [ T | \mathcal{H}_i ] \\
& \leq \sum_{i = 1}^{+ \infty} \delta^{i - 1} 2^i (\Phi \ln (2^i / \delta) + \Psi + 1) \\
& = O\left( \sum_{i = 1}^{+ \infty} i (0.2)^{i - 1}  \Phi \right) + O\left( \sum_{i = 1}^{+ \infty} (0.2)^{i - 1}  ( \Phi \ln \delta^{-1} + \Psi + 1) \right)  \\
& = O( \Phi \ln \delta^{-1} + \Psi ),
\end{align*}
where the inequality is due to \eqref{app:expected-time:eq:1} and the second equality is due to $\delta \leq 0.1$, which concludes the proof of the second part of Theorem~\ref{THM:EXPECTED-TIME-ALGO}.

\section{Calculation for Example~\ref{EXP:1}} \label{app:easy-example}

We first calculate the expected sample complexity of $\OptArmIdE(S, \delta)$ upon the input described in Example~\ref{EXP:1}. By Theorem~\ref{THM:EXPECTED-TIME-ALGO}, the number of samples used is upper bounded by
\begin{equation} \label{equ-1:example-0}
O \left( \sum_{i = 1}^n \left( \frac{\sigma_i^2}{\Delta_i^2} + \frac{1}{\Delta_i} \right) ( \ln \delta^{-1} + \ln (e + \ln \Delta_i^{-1} ) ) \right). 
\end{equation}
Note that 1) $\theta_i = 1 - \frac{i}{n}$ and for Bernoulli arms we have $\sigma_i^2 = \theta_i (1 - \theta_i) \leq 1 - \theta_i$ for all $i = 1, 2, \dots, n$; 2) $\sigma_1^2 \Delta_1^{-2} + \Delta_1^{-1} = O(\sigma_2^2 \Delta_2^{-2} + \Delta_2^{-1})$. Therefore, 
\begin{align*}
  \eqref{equ-1:example-0} & = O\left(\sum_{i=2}^n \left(\frac{i/n}{((i-1)/n)^2} + \frac{1}{(i-1)/n}\right)\left( \ln \delta^{-1} + \ln \left(e + \ln \frac{n}{i-1}\right) \right)   \right)\\
  &= O \left( \sum_{i = 2}^n \frac{n(2i - 1)}{(i - 1)^2}\left( \ln \delta^{-1} + \ln \left(e + \ln \frac{n}{i-1}\right) \right)\right) \nonumber  \\
 & = O \left( \sum_{i = 2}^n \frac{n}{i} \left( \ln \delta^{-1} + \ln \ln \frac{15 n}{i } \right) \right) .
\end{align*}
Note that $\sum_{i = 2}^n \frac{n}{i} \leq \int_{1}^{n} \frac{n}{x} d x$ and $\sum_{i = 2}^n \frac{n}{i} \ln\ln \frac{15 n}{i} \leq \int_{1}^{n} \frac{n}{x} \ln \ln \frac{15 n}{x} d x$. We further have
\begin{align}
 \eqref{equ-1:example-0}  & \leq O\left( \int_{1}^{n} \frac{n}{x} d x \cdot \ln \delta^{-1}+ \int_{1}^{{n} } \frac{n}{x} \ln \ln \frac{15 n}{x} d x \right)  
 = O \left( n \ln n ( \ln \delta^{-1} + \ln \ln n) \right) \nonumber.
\end{align}

In contrast, the expressions in the big-O notation in both \eqref{eq:gap-dependent-bound-1} and \eqref{eq:gap-dependent-bound-2} are lower bounded by 
\begin{align*}
\Omega \left(\sum_{i = 2}^{n} \Delta_i^{-2} \ln \delta^{-1} \right) = \Omega \left( \sum_{i = 2}^n \frac{n^2}{(i - 1)^2} \ln \delta^{-1} \right) = \Omega( n^2 \ln \delta^{-1}).
\end{align*}

\section{Missing Proofs in Section~\ref{SEC:VAR-ESTIMATE}} \label{app:proofs-in-sec:var-estimate}

\subsection{Proof of Lemma~\ref{LEMMA:VARIANCE-SAMPLING}}
Let $X_r$ be the random variable representing the $r$-th sample and $Y_r = \frac{(X_r - X_{r + T})^2}{2}$. Note that $\E Y_r = \sigma_i^2$ and $Y_r$'s are \textit{i.i.d.}. 

If $\sigma_i^2 \geq 2\tau$, we have 
\begin{multline*}
\Pr[ \hat{\sigma}_i^2 \leq \tau ] = \Pr \left[ \hat{\sigma}_i^2 \leq \left(1 - \frac{ \sigma_i^2 - \tau }{\sigma_i^2} \right) \sigma_i^2 \right] 
< \left( \frac{e^{ - \frac{ \sigma_i^2 - \tau }{\sigma_i^2}}}{ \left( \frac{\tau}{\sigma_i^2} \right)^{ \left( \frac{\tau}{\sigma_i^2} \right)}} \right)^{T \sigma_i^2}
\leq \delta^{\frac{ c(\sigma_i^2 - \tau) }{ \tau }} \cdot \left( \frac{\tau}{\sigma_i^2} \right)^{c \ln \delta^{-1}} \leq \delta \cdot  \left( \frac{\tau}{\sigma_i^2} \right)^c,
\end{multline*}
where the second inequality is due to Proposition \ref{prop:chernoffbound} and the last inequality is due to $\delta \leq e^{-1}$. Condition on event $\hat{\sigma}_i^2 > \tau$ which happens with probability at least $1 - \delta \cdot \left( \frac{\tau}{\sigma_i^2} \right)^c $, algorithm $\VarTest(i, \tau, \delta, c)$ outputs \true.

If $\sigma_i^2 \leq \tau / 2$, we have
\begin{multline*}
\Pr[ \hat{\sigma}_i^2 > \tau ] = \Pr \left[ \hat{\sigma}_i^2 > \left(1 + \frac{\tau - \sigma_i^2}{\sigma_i^2} \right) \sigma_i^2 \right] 
< \left( \frac{e^{\frac{\tau - \sigma_i^2}{\sigma_i^2}}}{ \left( \frac{\tau}{\sigma_i^2} \right)^{ \left( \frac{\tau}{\sigma_i^2} \right)}} \right)^{T \sigma_i^2}
\leq \delta^{\frac{c\tau}{\tau - \sigma_i^2}} \cdot \left( \frac{\sigma_i^2}{\tau} \right)^{c \ln \delta^{-1}} \leq \delta \cdot  \left( \frac{\sigma_i^2}{\tau} \right)^c,
\end{multline*}
where the second inequality is due to Proposition~\ref{prop:chernoffbound} and the last inequality is due to $\delta \leq e^{-1}$. Condition on event $\hat{\sigma}_i^2 \leq \tau$ which happens with probability at least $1 - \delta \cdot  \left( \frac{\sigma_i^2}{\tau} \right)^c $, algorithm $\VarTest(i, \tau, \delta, c)$ outputs \false.

Finally, it is straightforward to verify that the sample complexity is $\frac{2c}{\tau} \ln \delta^{-1}$. This concludes the proof of the lemma.

\subsection{Proof of Lemma~\ref{LEMMA:VARIANCE-EST}}

Suppose algorithm $\VarEst(i, \delta, \ell)$ terminates with $r = r_0$.

Consider the first claim. It is the easy to check that $\tau_{r_0} > \ell / 2$.  By Lemma~\ref{LEMMA:VARIANCE-SAMPLING},  $\VarTest(i, \tau_r, \delta / e , 80)$ uses $O( \frac{1}{\tau_r} \ln \delta^{-1} )$ samples. Hence, total samples is bounded by

\[
\sum_{r = 1}^{r_0} O\left( \frac{1}{\tau_r} \ln \delta^{-1} \right)  
= O\left(  \frac{1}{\tau_{r_0}} \ln \delta^{-1}   \right) =  O\left(  \frac{1}{\ell} \ln \delta^{-1} \right),
\]
where the last equality holds since $\tau_{r_0} > \ell / 2$ concluding the proof of the first claim.

Let $t$ be the smallest index such that $\sigma_i^2 > \tau_t$. Hence, $\sigma_i^2 \in (\tau_t, 2\tau_t]$. It is straightforward to verify the following facts: 1) for $r = 1, \dots, (t - 2)$, we have $\sigma_i^2 \leq 2\tau_t \leq \tau_r / 2$; 2) for $r \geq t + 1$, we have $\sigma_i^2 > \tau_t \geq 2\tau_r$. Let $\mathcal{E}$ denote the event when $r = 1, \dots, (t - 2)$, $\VarTest(i, \tau_r, \delta / e, 80) $ outputs \false, and when $r \geq t+1$,  $\VarTest(i, \tau_r, \delta / e, 80)$ outputs \true. By Lemma~\ref{LEMMA:VARIANCE-SAMPLING} and a union bound, we have
\[
\Pr[ \mathcal{E} ] \geq 1 - \frac{\delta}{e} \sum_{r = 1}^{t - 2} \frac{\sigma_i^2}{\tau_r} - \frac{\delta}{e} \sum_{r = t+1}^{+ \infty} \frac{\tau_r}{\sigma_i^2}   \geq 1 - \frac{\delta}{e} \sum_{r = 1}^{t - 2} \left( \frac{1}{2} \right)^r - \frac{\delta}{e} \sum_{r = t+1}^{+ \infty} \left( \frac{1}{2} \right)^{r - t}  \geq 1 - \delta.
\]

Now consider the second claim. Recall that $\sigma_i^2 \in (\tau_t, 2\tau_t]$. Given that $\sigma_i^2 \in (\ell, 1]$, we  have $\tau_t > \ell / 2$. Condition on event $\mathcal{E}$ which happens with probability at least $1 - \delta$. $\VarEst(i, \delta, \ell)$ stops with $r = t - 1, t$ or $t + 1$. Therefore, we have $\sigma_i^2 \in (\tau_t, 2\tau_t] \subset (\tau_{r_0} / 2, 4\tau_{r_0}]$, which means $\tau = \tau_{r_0} \in [\sigma_i^2 / 4, 2\sigma_i^2)$. Moreover, the sample complexity is bounded by 
\[
\sum_{r = 1}^{t + 1} O\left( \frac{1}{\tau_r} \ln \delta^{-1} \right)  
= O\left(  \frac{1}{\tau_{t+1}} \ln \delta^{-1}   \right) =  O\left(  \frac{1}{\sigma_i^2} \ln \delta^{-1} \right),
\]
where the last equality is due to $\sigma_i^2 = \Theta( \tau_{t + 1} )$.

Suppose $\VarEst(i, \delta, \ell)$ terminates with $\tau_{r_0} \geq 2 \sigma_i^2 $ which means $\VarTest(i, \tau_r, \delta / e , 80)$ outputs \true when $r = r_0$. Recall that $\sigma_i^2 \leq \tau_{r_0} / 2$. By~Lemma~\ref{LEMMA:VARIANCE-SAMPLING}, the probability that this event happens is no greater than 

\[ 
\delta \cdot \left( \frac{\sigma_i^2}{\tau_{r_0} } \right)^{80} =  \delta \cdot 2^{- 80 \log_2 \left( \frac{\tau_{r_0}}{\sigma_i^2}  \right)}\leq \delta \cdot 2^{-40r_m},
\]
where the last inequality is due to $r_m = \left\lceil \log_2 \left( \frac{\tau_{r_0}}{\sigma_i^2}  \right) \right\rceil \leq 2 \log_2 \left( \frac{\tau_{r_0}}{\sigma_i^2}  \right)$ when $\frac{\tau_{r_0}}{\sigma_i^2} \geq 2$.
On the other hand, suppose $\VarEst(i, \delta, \ell)$ terminates with $\tau_{r_0} < \sigma_i^2 /4 $ which means $\VarTest(i, \tau_r, \delta / e , 80)$ outputs \false when $r = r_0 - 1$.
Since $\sigma_i^2 > 2 \tau_{r_0 - 1}$, by Lemma~\ref{LEMMA:VARIANCE-SAMPLING}, the probability that this event happens is bounded by
\[
 \delta \cdot \left( \frac{\tau_{r_0 - 1} }{\sigma_i^2} \right)^{80} = \delta \cdot 2^{- 80 \left( \log_2 \left( \frac{\sigma_i^2}{\tau_{r_0}}  \right) - 1 \right) } \leq \delta \cdot 2^{-20r_m},
\]
where the last inequality is due to $r_m = \left\lceil \log_2 \left( \frac{\sigma_i^2}{\tau_{r_0}}  \right) \right\rceil \leq 4 \left( \log_2 \left( \frac{\sigma_i^2}{\tau_{r_0}}  \right) - 1 \right)$ when $\frac{\sigma_i^2}{\tau_{r_0}} > 4$ concluding the proof of the second claim. 

For the last claim, recall that $\tau = \tau_{r_0} \geq 2\ell > \ell$, which means $\VarTest(i, \tau_r, \delta / e , 80)$ outputs \true when $r = r_0$. Also, we have $\tau = \tau_{r_0} > 2 \sigma_i^2$. Using the same way as that in the proof of the second claim, this claim can also be proved. 

\subsection{Proof of Lemma~\ref{LEMMA:MEAN-EST}}

Let $\mathcal{E}_1$ be the event that $\hat{\sigma}_i^2 \geq  \sigma_i^2 / 4$. According to Lemma~\ref{LEMMA:VARIANCE-EST:b}, we have $\Pr[ \mathcal{E}_1 | \sigma_i^2 \in (\ell, 1] ] \geq 1 - \delta/2$. Also note that when $\sigma_i^2 \leq \ell$, there is $\hat{\sigma}_i^2 > \ell / 2 \geq \sigma_i^2 / 4$. Hence, it holds that $\Pr[ \mathcal{E}_1 ] \geq 1 - \delta / 2$. 

Condition on event $\mathcal{E}_1$ which happens with probability at least $1 - \delta / 2$. Let $\mathcal{E}_2$ be the event that $| \hat{\theta}_i - \theta_i | \leq \eps$ and the number of samples used at Line~3 of $\MeanEst(i, \eps, \delta)$ is bounded by $O
\left( \left( \frac{\sigma_i^2}{\eps^2} + \frac{1}{\eps} \right) \ln \delta^{-1} \right)$. According to event $\mathcal{E}_1$ and Proposition~\ref{prop:bernstein}, it holds that with probability at least $1 - \delta / 2$, event $\mathcal{E}_2$ happens, which means $\Pr[ \mathcal{E}_2 | \mathcal{E}_1 ] \geq 1 - \delta / 2$. Therefore,

\[
\Pr[ \mathcal{E}_2] \geq \Pr[ \mathcal{E}_1 \wedge \mathcal{E}_2 ] = \Pr[ \mathcal{E}_2 | \mathcal{E}_1 ] \cdot \Pr[ \mathcal{E}_1 ] = (1 - \delta / 2) \cdot (1 - \delta / 2) \geq 1 - \delta.
\]

Conditioning on event $ \mathcal{E}_2$ which happens with probability at least $1 - \delta$, we have $| \hat{\theta}_i - \theta_i | \leq \eps$ and the number of samples used at Line~3 of $\MeanEst(i, \eps, \delta)$ is $O
\left( \left( \frac{\sigma_i^2}{\eps^2} + \frac{1}{\eps} \right) \ln \delta^{-1} \right)$. Also note that the number of samples used at Line~2 of $\MeanEst(i, \eps, \delta)$ is always bounded by $O( \frac{1}{\eps} \ln \delta^{-1} )$ by Lemma~\ref{LEMMA:VARIANCE-EST:a}. Therefore, this lemma is proved.

\subsection{Proof of Lemma~\ref{LEMMA:MEAN-EST-COROL}}

According to Lemma~\ref{LEMMA:VARIANCE-EST:a}, Line~2 of  $\MeanEst(i, \eps, \delta)$ uses at most $O\left( \frac{1}{\eps} \ln \delta^{-1} \right)$ samples. We bound the number of samples used at Line~\ref{alg:mean-est-line3} by observing $\hat{\sigma}_i^2 \leq 1$. Hence, the first claim of the lemma is proved.

For the second claim, according to Lemma~\ref{LEMMA:VARIANCE-EST:b} and \ref{LEMMA:VARIANCE-EST:c}, there is $\Pr[ \hat{\sigma}_a^2 = 2^{-k} ] \leq \delta \cdot 2^{-20 r_m }  $ for $2^{-k} \geq \max\{ 3\sigma_i^2, 2\eps \}$. Via a union bound,  with probability at least $1 - \sum_{k > j} \delta \cdot 2^{-20k} \geq 1 - \delta \cdot 2^{-20j}$, it holds that $\hat{\sigma}_i^2 \leq j \sigma_i^2 \leq \max\{ j \sigma_i^2, 2\eps \}$ for $j \geq \max\{ 3, \frac{2\eps}{\sigma_i^2} \} $. Note that when $3 \leq j < \frac{2\eps}{\sigma_i^2}$, $\hat{\sigma}_i^2  \leq \max\{ j \sigma_i^2, 2\eps \}$ holds with probability at least $1 - \delta \cdot 2^{-20 \frac{2\eps}{\sigma_i^2} } \geq 1 - \delta \cdot 2^{-20j}$. Hence, Line~3 of  $\MeanEst(i, \eps, \delta)$ uses $O \left(  \left( \frac{ \max\{ j \sigma_i^2, 2\eps \} }{\eps^2} + \frac{1}{\eps} \right) \ln \delta^{-1} \right)  = O \left(  \left( \frac{ j\sigma_i^2 }{\eps^2} + \frac{1}{\eps} \right) \ln \delta^{-1} \right) $ samples with probability at least $1 - \delta \cdot 2^{-20j}$ for $j \geq 3 $.  Therefore, for $j \geq 3$, with probability at least $1 - \delta \cdot 2^{-20j}$, we have $Q \leq O \left(  \left( \frac{ j \sigma_i^2 }{\eps^2} + \frac{1}{\eps} \right) \ln \delta^{-1} \right) + O\left( \frac{1}{\eps} \ln \delta^{-1} \right) = O \left(  \left( \frac{j \sigma_i^2 }{\eps^2} + \frac{1}{\eps} \right) \ln \delta^{-1} \right)$ concluding the proof of the second claim of this lemma.

\section{Missing Materials in Section~3} \label{app:mis-material-in-sec:naive-way}

\subsection{Algorithm $\NaiveOptArmId$}

\begin{algorithm} 
\DontPrintSemicolon
\caption{Naive Best Arm Identification, $\NaiveOptArmId(S, \delta)$}
\label{alg:naive-opt-arm-id}
 \textbf{Input:} Arm set $S$ and confidence level $\delta$\\
 $S_1 \leftarrow S$, $r \leftarrow 1$\\
\While {$|S_r| > 1$} {
 Set $\eps_r \leftarrow 1 / 2^r$ and $\delta_r \leftarrow 1 / (2r^2) \cdot \delta$\\
\lFor {$i \in S_r$} {
 $\hat{\theta}_i^r \leftarrow \MeanEst(i, \eps_r / 2, \delta_r / |S_r| )$
}
 Let $a_r = \argmax_{i \in S_r} \hat{\theta}_i^r $\\
 $S_{r + 1} \leftarrow S_r \backslash \{ i \in S_r | \hat{\theta}_i^r < \hat{\theta}_{a_r}^r - \eps_r \}$\\
$r \leftarrow r + 1$
}
 \textbf{Output:} The remaining arm in $S_r$
\end{algorithm}

\subsection{Proof of Theorem~\ref{THM:NAIVE-OPT-ARM-ID}}

Let $\mathcal{E}_i^r$ denote the event that $|\hat{\theta}_i^r - \theta_i | \leq \eps_r / 2$ and the sample complexity of algorithm $\MeanEst(i, \eps_r / 2, \delta_r / |S_r| )$ is $O \left( \left( \frac{\sigma_i^2}{\eps_r^2} + \frac{1}{\eps_r} \right) \ln \frac{|S_r|}{\delta_r} \right)$. By Lemma~\ref{LEMMA:MEAN-EST}, we have $\Pr[ \mathcal{E}_i^r ] \geq 1 - \delta_r / |S_r|$. Let $\mathcal{E}^r $ be the event $\bigwedge_{i \in S_r} \mathcal{E}_i^r$. Via a union bound, we have $\Pr[ \mathcal{E}^r ]  \geq 1 - \delta_r$. Let $\mathcal{E}$ denote the event $\bigwedge_{r = 1}^{+\infty} \mathcal{E}^r$. Again via a union bound, we can get $\Pr[ \mathcal{E} ] \geq 1 - \sum_{r = 1}^{+\infty} \delta_r = 1 - \delta \cdot \sum_{r = 1}^{+\infty} 1 / (2r^2) \geq 1 - \delta$.

Condition on event $\mathcal{E}$ which happens with probability at least $1 - \delta$.

First, we claim that the best arm always survives i.e., $S_{[1]} \in S_r$. Suppose arm $S_{[1]}$ survives after round $r = k$. In round $r = k + 1$, we have $\hat{\theta}_{a_r}^r \leq \theta_{a_r} + \eps_r / 2 \leq \theta_{S_{[1]}} + \eps_r / 2 \leq \hat{\theta}_{S_{[1]}}^r + \eps_r$, which means arm $S_{[1]}$ is not eliminated after round $r = k + 1$. Note that $S_{[1]} \in S$. Therefore, this claim is proved.

Let $t_i$ be the smallest index such that $\Delta_i > \eps_{t_i}$. Define $\eps_0 = 1$. Hence $\Delta_i \in (\eps_{t_i}, \eps_{t_i - 1}]$. Next, we claim that arm $S_{[i]}, i \neq 1$ is eliminated before round $r = t_i + 1$ finishes. Suppose arm $S_{[i]}$ survives after round $r = t_i$ finishes. Consider round $r = t_i + 1$. According to the first claim, we know that $S_{[1]} \in S_{t_i + 1}$. Also, we can find that $\hat{\theta}_i^{t_i + 1} \leq \theta_i + \eps_{t_i + 1} / 2 < \theta_{S_{[1]}} - \eps_{t_i}  + \eps_{t_i + 1} / 2  = (\theta_{S_{[1]}} - \eps_{t_i + 1} / 2) - \eps_{t_i + 1} \leq \hat{\theta}_{S_{[1]}}^{t_i + 1} - \eps_{t_i + 1} \leq \hat{\theta}_{a_r}^{t_i + 1} - \eps_{t_i + 1}$, which means after round $r = t_i + 1$ finishes arm $S_{[i]}, i \neq 1$ must be eliminated. Therefore, this claim is also proved.

Above all, we have proved that $S_{t_2 + 1} = \{ S_{[1]} \}$ and hence the best arm in $S$ is output after round $r = t_2 + 1$ finishes. Thus, the first part of this lemma is proved.

By event $\mathcal{E}_i^r$, the sample complexity of algorithm $\MeanEst(i, \eps_r / 2, \delta_r / |S_r| )$ is $O \left( \left( \frac{\sigma_i^2}{\eps_r^2} + \frac{1}{\eps_r} \right) \ln \frac{|S_r|}{\delta_r} \right)$. Therefore, the number of samples used for arm $S_{[i]}$ is bounded by 
\begin{multline*}
\sum_{r = 1}^{t_i + 1} O \left( \left( \frac{\sigma_i^2}{\eps_r^2} + \frac{1}{\eps_r} \right) \ln \frac{|S_r|}{\delta_r} \right) 
=  \sum_{r = 1}^{t_i + 1} O \left( \left( \frac{\sigma_i^2}{ \eps_r^2 } + \frac{1}{\eps_r} \right) ( \ln \delta^{-1}  + \ln r + \ln |S_r| ) \right) \\
= O \left( \left( \frac{\sigma_i^2}{ \eps_{t_i + 1}^2 } + \frac{1}{\eps_{t_i + 1}} \right) ( \ln \delta^{-1}  + \ln (t_i + 1) + \ln |S| ) \right) \\
= O \left( \left( \frac{\sigma_i^2}{ \Delta_i^2 } + \frac{1}{\Delta_i} \right) ( \ln \delta^{-1} + \ln \ln \Delta_i^{-1} + \ln |S| ) \right),
\end{multline*}
where the last equality is due to $\eps_{t_i + 1} = \Theta( \Delta_i )$ and $t_i + 1 = \Theta( \ln \Delta_i^{-1} ) $. Finally, the total sample complexity equals to the summation of those for every arm in $S$.

\subsection{Proof of Corollary~\ref{COL:NAIVE-OPT-ARM-ID}}

The algorithm can be derived by running the \textbf{while} loop in algorithm $\NaiveOptArmId(S, \delta)$ for at most $O( \ln \eps^{-1} )$ rounds and then randomly output an arm in $S_r$. 

\section{Missing Proofs in Section~\ref{SEC:EPS-OPTIMAL-ARM}} \label{app:proofs-in-SEC:EPS-OPTIMAL-ARM}

\subsection{Proof of Lemma~\ref{LEMMA:PARTITION-OF-S}}
According to Lemma \ref{LEMMA:VARIANCE-EST:a}, we have that for any arm $i \in S$, $\hat{\sigma}_i^2 > \eps / 2$. Also note that $l(\hat{B}_N) = 2^{-N} \in (\eps/4, \eps/2]$. Therefore, every arm $i$ belongs to one of the buckets $\{\hB_j | j = 1, 2, \dots, N\}$.

\subsection{Proof of Lemma~\ref{LEMMA:HE-CARD-CUT}}
When $| \hat{B}_j |$ is even, half of $ \hat{B}_j $ is deleted. Hence, we have $|T_j| \leq \frac{1}{2} | \hat{B}_j |$. When $| \hat{B}_j |$ is odd, suppose $| \hat{B}_j | = 2k + 1$ where $k \geq 1$, $\frac{k}{2k+1} = \frac{1}{2} (1 - \frac{1}{2k+1}) \geq \frac{1}{3}$ of $\hat{B}_j$ is deleted. Hence, we have  $|T_j| \leq \frac{2}{3} | \hat{B}_j |$. 

\subsection{Proof of Lemma~\ref{LEMMA:HE-CARD-R}}
Just note that the number of non-empty buckets (i.e.\ $\hat{B}_j$'s) is no greater than $N = O( \ln \eps^{-1})$. 

\subsection{Proof of Lemma~\ref{LEMMA:HE-EVENT}}

Let $Z_{ij}^a$ be the random variable $\mathds{1} \{ a \in B_i \wedge a \in \hat{B}_j \}$. We only consider those with $|i - j| \geq 3$. 

When $i \leq N - 2$, we have $\sigma_a^2 > l( B_{N -2} ) = 2^{-N + 2} > \eps$, for any  $a \in B_i$. What's more, for any two arms $a \in B_i$ and $b \in \hat{B}_j$, it holds that $\frac{ \max \{ \sigma_a^2,  \hat{\sigma}_b^2 \} } {\min \{ \sigma_a^2,  \hat{\sigma}_b^2 \} } > 4 $. By Lemma~\ref{LEMMA:VARIANCE-EST:b}, we have 
\[
\Pr[ Z_{ij}^a = 1] \leq \frac{\delta}{2N^2} \cdot 2^{-20 r_m } \leq \frac{\delta}{N^2} \cdot 2^{-20 |i - j| } ,
\]
where the last inequality is due to $r_m = \lceil | \log_2 \frac{\sigma_a^2 }{2^{-j + 1}} | \rceil \geq |i - j| - 1$. When $i \geq N - 1$, we have $\sigma_a^2 \leq u( B_{N - 1} ) \leq 2\eps, a \in B_i$. What's more, note that $j \leq i - 3 \leq N - 3 $. Hence $\hat{\sigma}_a^2 > 4 \sigma_a^2 $ and $\hat{\sigma}_a^2 > l(\hat{B}_{N - 3}) > 2\eps $. By Lemma~\ref{LEMMA:VARIANCE-EST:c} and using the same argument, we can also get  $\Pr[ Z_{ij}^a = 1] \leq  \frac{\delta}{N^2} \cdot 2^{-20 |i - j| }$. Above all, we obtain 
\begin{equation} \label{LEMMA:HE-EVENT-equ1}
\E[ Z_{ij}^a ] \leq \Pr[ Z_{ij}^a = 1] \leq  \frac{\delta}{N^2} \cdot 2^{-20 |i - j| }.
\end{equation}

Note that $| B_i \cap \hat{B}_j | = \sum_{a \in B_i} Z_{ij}^a$. By Markov's Inequality, there is
\begin{multline*}
\Pr\left[ | B_i \cap \hat{B}_j | \geq |B_i| \cdot 2^{- 10|i - j|} \cdot N^{-1} \right ] = \Pr \left[ \sum_{a \in B_i} Z_{ij}^a \geq |B_i| \cdot 2^{- 10|i - j|} \cdot N^{-1} \right] \\
\leq \frac{\E\left[ \sum_{a \in B_i} Z_{ij}^a \right]}{|B_i| \cdot 2^{-10 |i - j|} \cdot N^{-1} } = \frac{ \sum_{a \in B_i} \E[  Z_{ij}^a]}{|B_i| \cdot 2^{-10 |i - j|} \cdot N^{-1} } \leq  \frac{\delta}{N} \cdot 2^{-10 |i - j| },
\end{multline*}
where the last inequality is due to (\ref{LEMMA:HE-EVENT-equ1}).

Therefore, via a union bound, we obtain
\[
\Pr[ \mathcal{E} ] \leq \sum_{j} \sum_{i, | i - j | \geq 3} \frac{\delta}{N} \cdot 2^{-10 |i - j| } = \sum_{j} \frac{\delta}{N}  \sum_{i, | i - j | \geq 3} 2^{-10 |i - j| } \leq \sum_j  \frac{\delta / 3}{N} \leq \delta / 3,
\]
where the last inequality is due to the number of buckets $\hat{B}_j$ is no greater than $N$.

\subsection{Proof of Lemma~\ref{LEMMA:HE-SITUA1}}

Our goal is to give a constant upper bound on $ \frac{\sum_{a \in T_j} \sigma_a^2  }{ \sum_{a \in \hat{B}_j} \sigma_a^2 } $.

Define $o( \hat{B}_j )$ to be the set $\bigcup_{i, |i-j| \geq 3} (B_i \cap \hat{B}_j)$ and $n( \hat{B}_j )$ to be the set $\bigcup_{i, |i-j| \leq 2} (B_i \cap \hat{B}_j)$. It is straightforward to verify that $n( \hat{B}_j ) \cap o( \hat{B}_j ) = \emptyset $ and $ n( \hat{B}_j ) \cup o( \hat{B}_j ) = \hat{B}_j $, which means $n( \hat{B}_j ), o( \hat{B}_j )$ is a partition of set $\hat{B}_j$. 

Note that 
\begin{equation} \label{equ:HE-situa1-0}
|o(\hat{B}_j)| = \sum_{i, |i - j| \geq 3} |B_i \cap \hat{B}_j | \leq | \hat{B}_j | \cdot  \sum_{i, |i - j| \geq 3} 2^{-5 |i - j| } \leq \frac{ | \hat{B}_j |} {1024},
\end{equation}
and $|T_j| \leq \frac{2}{3} | \hat{B}_j |$ by Lemma~\ref{LEMMA:HE-CARD-CUT}. We have that at least $\left( \frac{1}{3} - \frac{1}{1024} \right) | \hat{B}_j | $ arms in $n( \hat{B}_j )$ are discarded. Also since $\sigma_a^2 \geq \max_{a\in n( \hat{B}_j )} \sigma_a^2 \cdot 2^{-5} $ for any arm $a \in n(\hat{B}_j) $, we have
\begin{align}  \label{equ:HE-situa1-1}
\sum_{a \in \hat{B}_j} \sigma_a^2 - \sum_{a \in T_j} \sigma_a^2 \geq   \left( \frac{1}{3} - \frac{1}{1024} \right) n( \hat{B}_j ) \cdot \max_{a\in n( \hat{B}_j )} \sigma_a^2 \cdot 2^{-5} 
\geq 2^{-5} \left( \frac{1}{3} - \frac{1}{1024} \right) \cdot \sum_{a \in n(\hat{B}_j)} \sigma_a^2 .
\end{align}

Next, we would like to derive a lower bound on $\sum_{a \in n(\hat{B}_j)} \sigma_a^2$. Note that
\begin{multline*}
\sum_{a \in o( \hat{B}_j )} \sigma_a^2 = \sum_{i, |i - j| \geq 3} \sum_{a \in B_i \cap \hat{B}_j} \sigma_a^2 \leq  \sum_{i, |i - j| \geq 3} |B_i \cap \hat{B}_j | \cdot u( B_i ) \\ 
\leq \sum_{i, |i - j| \geq 3} | \hat{B}_j | \cdot 2^{-5|i - j|} \cdot u( B_i ) = | \hat{B}_j | \cdot \sum_{i, |i - j| \geq 3} 2^{-5|i-j|} \cdot 2^{-i + 1},
\end{multline*}
and 
\[
\sum_{a \in n( \hat{B}_j )} \sigma_a^2 \geq | n(\hat{B}_j) | \cdot l( \hat{B}_{j+2} ) = \frac{1023}{1024}  | \hat{B}_j| \cdot 2^{-j - 2},
\]
where the first inequality is due to (\ref{equ:HE-situa1-0}).
Therefore, we can get
\begin{equation*}
\frac{\sum_{a \in o( \hat{B}_j )} \sigma_a^2}{\sum_{a \in n( \hat{B}_j )} \sigma_a^2} \leq \frac{ | \hat{B}_j | \sum_{i, |i - j| \geq 3} 2^{-5|i-j|} \cdot 2^{-i + 1} }{ \frac{1023}{1024}  | \hat{B}_j| \cdot 2^{-j-2} } \\
\leq 8 \cdot \frac{1024}{1023} \sum_{i, |i - j| \geq 3} 2^{-4|i-j|} \leq \frac{1}{64},
\end{equation*}
which leads to $ \sum_{a \in n( \hat{B}_j )} \sigma_a^2  \geq \frac{64}{65} \sum_{a \in \hat{B}_j} \sigma_a^2$. 

Plugging this relation between $\sum_{a \in n( \hat{B}_j )} \sigma_a^2$ and $\sum_{a \in \hat{B}_j} \sigma_a^2$ into (\ref{equ:HE-situa1-1}), we have 
\[
\sum_{a \in \hat{B}_j} \sigma_a^2 - \sum_{a \in C_j} \sigma_a^2 \geq 2^{-5} \left( \frac{1}{3} - \frac{1}{1024} \right)  \frac{64}{65} \sum_{a \in \hat{B}_j} \sigma_a^2 \geq 2^{-7} \sum_{a \in \hat{B}_j} \sigma_a^2.
\]
Therefore, 
\[
\frac{\sum_{a \in T_j} \sigma_a^2  }{ \sum_{a \in \hat{B}_j} \sigma_a^2 } \leq  1 - 2^{-7}  = \frac{127}{128}.
\]

\subsection{Proof of Corollary~\ref{COROL:HE-SITUA1}}

According to Lemma~\ref{LEMMA:HE-CARD-CUT} and Lemma~\ref{LEMMA:HE-SITUA1}, we know that $\frac{ \sum_{a \in T_j} \sigma_a^2 }{ \sum_{a \in \hat{B}_j} \sigma_a^2 }  \leq \frac{127}{128} $ and $\frac{|T_j|}{|\hat{B}_j|} \leq \frac{2}{3}$. Hence,
\[
\frac{ \sum_{a \in T_j} ( \sigma_a^2 + \eps ) }{ \sum_{a \in \hat{B}_j} ( \sigma_a^2 + \eps ) } = \frac{\sum_{a \in T_j} \sigma_a^2 +  |T_j| \eps}{ \sum_{a \in \hat{B}_j} \sigma_a^2 + | \hat{B}_j | \eps } \leq \max \left\{ \frac{ \sum_{a \in T_j} \sigma_a^2 }{ \sum_{a \in \hat{B}_j} \sigma_a^2 }, \frac{|T_j|}{|\hat{B}_j|} \right\} \leq \frac{127}{128},
\]
which concludes the proof of this corollary.

\subsection{Proof of Lemma~\ref{LEMMA:HE-SITUA2}}

The key part is to prove $\sum_{a \in \hat{B}_j } \sigma_a^2 = O( | \hat{B}_j | \eps )$. Note that $\sum_{a \in \hat{B}_j } \sigma_a^2 = \sum_{a \in n( \hat{B}_j ) } \sigma_a^2 + \sum_{a \in o( \hat{B}_j ) } \sigma_a^2$. We bound $\sum_{a \in n( \hat{B}_j ) } \sigma_a^2$ and $\sum_{a \in o( \hat{B}_j ) } \sigma_a^2$ respectively.

It is easy to see
\begin{equation}  \label{LEMMA:HE-SITUA1-eq1}
\sum_{a \in n( \hat{B}_j ) } \sigma_a^2 \leq | \hat{B}_j | \cdot u( B_{j - 2} ) = | \hat{B}_j | \cdot 2^{-j + 3} \leq | \hat{B}_j | \cdot 2^{-N + 5} \leq 16 | \hat{B}_j | \eps.
\end{equation}
Also, we have 
\begin{multline}  \label{LEMMA:HE-SITUA1-eq2}
\sum_{a \in o( \hat{B}_j )} \sigma_a^2 = \sum_{i, i \leq j - 3} \sum_{a \in B_i \cap \hat{B}_j} \sigma_a^2 \leq  \sum_{i, i \leq j - 3 } |B_i \cap \hat{B}_j | \cdot u( B_i ) \\ 
\leq \sum_{i, i \leq j - 3} | \hat{B}_j | \cdot 2^{-5|i - j|} \cdot u( B_i ) = | \hat{B}_j | \cdot \sum_{i, i \leq j - 3 } 2^{-5(j -i)} \cdot 2^{-i + 1} \\
= | \hat{B}_j | \cdot 2^{-j + 1} \cdot \sum_{i, i \leq j - 3 } 2^{-4(j -i)} \leq | \hat{B}_j | \cdot 2^{- N + 3} \cdot \sum_{i, i \leq j - 3 } 2^{-4(j -i)} \leq \frac{ | \hat{B}_j | \eps } {128}.
\end{multline}
Hence, by (\ref{LEMMA:HE-SITUA1-eq1}) and (\ref{LEMMA:HE-SITUA1-eq2}), we have $\sum_{a \in \hat{B}_j } \sigma_a^2 \leq 17 | \hat{B}_j | \eps$.

Therefore,
\[
\frac{ \sum_{a \in T_j} ( \sigma_a^2 + \eps )  }{ \sum_{a \in \hat{B}_j} ( \sigma_a^2 + \eps ) } \leq \frac{ \sum_{a \in \hat{B}_j}  \sigma_a^2 + |T_j| \eps   }{ \sum_{a \in \hat{B}_j}  \sigma_a^2 + |\hat{B}_j| \eps  } \leq \frac{ 17 | \hat{B}_j | \eps + |T_j| \eps}{17 | \hat{B}_j | \eps + |\hat{B}_j| \eps } \leq \frac{127}{128},
\] 
where the second last inequality is due to $|T_j| \leq |\hat{B}_j|$ and the last inequality is due to $|T_j| \leq \frac{2}{3} |\hat{B}_j|$ by Lemma \ref{LEMMA:HE-CARD-CUT}. 

\subsection{Proof of Lemma~\ref{LEMMA:HE_SITUA3}}

Note that $\sum_{ a \in \hat{B}_j } \sigma_a^2 = \sum_{ a \in n( \hat{B}_j ) } \sigma_a^2 + \sum_{ a \in o (\hat{B}_j) } \sigma_a^2 $.

First, we give an upper bound on $\sum_{ a \in n( \hat{B}_j ) } \sigma_a^2$. Let $i^*$ be the smallest index such that $B_{i^*}$ pollutes $\hat{B}_j$. Hence $i^* \leq N - 1$. Since $\hat{B}_j$ is polluted by $B_{i^*}$ where $|i - j| \geq 3$, we have $| B_{i^*} \cap \hat{B}_j | > | \hat{B}_j | \cdot 2^{-5 |i^* - j|} $. Also by event $\mathcal{E}$, we have that $| B_{i^*} \cap \hat{B}_j | < | B_{i^*} | \cdot 2^{-10 |i^* - j|} \cdot N^{-1}$. Therefore, it holds that $| \hat{B}_j | \leq | B_{i^*} | \cdot 2^{-5 |i^* - j|} \cdot N^{-1}$. Hence,
\begin{equation} \label{equ:HE-situa3-1}
\frac{ \sum_{ a \in n( \hat{B}_j ) } \sigma_a^2}{ \sum_{a \in S} \sigma_a^2 } \leq \frac{ | \hat{B}_j | \cdot u(B_{j - 2}) } { \sum_{a \in B_{i^*}} \sigma_a^2 } \leq \frac{ | B_i | \cdot 2^{-5 |i^* - j|} \cdot N^{-1} \cdot 2^{ - j + 3} } { |B_i| \cdot l(B_{i^*}) } \leq 2^{-9} \cdot N^{-1}.
\end{equation}

Then we give an upper bound on $\sum_{ a \in o (\hat{B}_j) } \sigma_a^2 $. For any $|i - j| \geq 3$, we have 
\[
\frac{ \sum_{a \in B_i \cap \hat{B}_j } \sigma_a^2 }{ \sum_{a \in B_i} \sigma_a^2} \leq \frac{ | B_i \cap \hat{B}_j | \cdot u(B_i) }{ | B_i| \cdot l(B_i)} \leq 2^{-10 |i - j| + 1} \cdot N^{-1} \leq 2^{-29} \cdot N^{-1}.
\] 
Hence, we have
\begin{equation} \label{equ:HE-situa3-2}
\frac{ \sum_{ a \in o (\hat{B}_j) } \sigma_a^2 }{ \sum_{ a \in S }\sigma_a^2  } \leq \frac{ \sum_{i, |i - j| \geq 3} \sum_{ a \in o (\hat{B}_j) } \sigma_a^2  } { \sum_{i, |i - j| \geq 3} \sum_{ a \in B_j } \sigma_a^2  } \leq \max_{i, |i -j| \geq 3} \left\{ \frac{ \sum_{a \in B_i \cap \hat{B}_j } \sigma_a^2 }{ \sum_{a \in B_i} \sigma_a^2} \right\} \leq 2^{-29} \cdot N^{-1}.
\end{equation}

Combining (\ref{equ:HE-situa3-1}) and (\ref{equ:HE-situa3-2}), we prove this lemma.

\subsection{Proof of Lemma~\ref{LEMMA:HE_SITUA4}}

Recall that $\sum_{ a \in \hat{B}_j } \sigma_a^2 = \sum_{ a \in n( \hat{B}_j ) } \sigma_a^2 + \sum_{ a \in o (\hat{B}_j) } \sigma_a^2 $.

For $\sum_{ a \in n( \hat{B}_j ) } \sigma_a^2$, we have
\begin{multline*}
\sum_{ a \in n( \hat{B}_j ) } \sigma_a^2 \leq | \hat{B}_j| \cdot u(B_{j-2}) \leq |B_N| \cdot 2^{-5|N-j|} \cdot N^{-1} \cdot 2^{-j + 3} 
\leq |B_N| \cdot 2^{-4|N-j|} \cdot N^{-1}  \cdot 2^{-N} \cdot 8 \leq 2^{-10} \cdot N^{-1} \cdot | S | \eps.
\end{multline*}

For $\sum_{ a \in o (\hat{B}_j) } \sigma_a^2$, from the proof of Lemma \ref{LEMMA:HE_SITUA3}, it holds that $\frac{ \sum_{ a \in o (\hat{B}_j) } \sigma_a^2 }{ \sum_{ a \in S }\sigma_a^2  } \leq 2^{-29} \cdot N^{-1} $. Hence,
\[
\frac{ \sum_{ a \in \hat{B}_j } \sigma_a^2 }{ \sum_{ a \in S }  (\sigma_a^2 + \eps) } = \frac{ \sum_{ a \in o(\hat{B}_j) } \sigma_a^2 + \sum_{ a \in n(\hat{B}_j) } \sigma_a^2}{ \sum_{ a \in S } \sigma_a^2 + |S| \eps} \leq \max \left\{ \frac{\sum_{ a \in o(\hat{B}_j) } \sigma_a^2}{ \sum_{ a \in S } \sigma_a^2 }, \frac{\sum_{ a \in n(\hat{B}_j) } \sigma_a^2}{ |S| \eps } \right\} = 2^{-10} \cdot N^{-1}.
\]

\subsection{Correctness and Sample Complexity of the Procedure $\HE$}\label{sec:sample-complexity-HE}

In this subsection, we prove the second and the third items of Theorem~\ref{thm:HE}. We first introduce a helper lemma as follows.

\begin{lemma} \label{lemma:sort-bound}
Given that $|\widehat{\theta}_i - \theta_i| \leq \Delta$ for every arm $i \in S$. Let $\theta_{p_1} \geq \theta_{p_2} \geq \cdots \geq \theta_{p_{|S|}}$ be the sorted sequence of $\theta_i$'s. Also, let $\widehat{\theta}_{q_1} \geq \widehat{\theta}_{q_2} \geq \cdots \geq \widehat{\theta}_{q_{|S|}}$ be the sorted sequence of $\widehat{\theta}_i$'s. Then for every index $t \in [|S|]$, we have $|\widehat{\theta}_{q_t} - \theta_{p_t}| \leq \Delta$.
\end{lemma}

\begin{proof}
Suppose this lemma does not hold. Let $t^*$ be the smallest index such $|\widehat{\theta}_{q_{t^*}} - \theta_{p_{t^*}}| > \Delta$. If $\widehat{\theta}_{q_{t^*}} < \theta_{p_{t^*}} - \Delta$, there are at least $t^*$ arms with means in interval $[0, \theta_{p_{t^*}})$, which is a contradiction since  there are only $t^* - 1$ arms with means in that interval.  If $\widehat{\theta}_{q_{t^*}} > \theta_{p_{t^*}} + \Delta$, there are at least $n - t^* + 1$ arms with means in interval $(\theta_{p_{t^*}}, 1]$, which is a contradiction since  there are only $n - t^*$ arms with means in that interval. Hence, the assumption is wrong, and this lemma is proved.
\end{proof}

The following lemma proves the second item of Theorem~\ref{thm:HE}.

\begin{lemma} \label{lemma:HE-correctness}
With probability at least $1 - \delta / 3$, we have $ | \theta_{(T \cup R)_{[1]}} - \theta_{S_{[1]}} | \leq \eps$.
\end{lemma}

\begin{proof}
According to Lemma \ref{LEMMA:MEAN-EST}, with probability at least $1 - \delta / (9N) $, it holds that $| \hat{\theta}_i - \theta_i | \leq \eps / 2$. By Lemma 1 from \cite{even2002pac}, there is $ \Pr[ \theta_{ (T_j)_{[1]}} \geq \theta_{ (\hat{B}_j)_{[1]} } - \eps ] \geq 1 - \delta/ (3N)$. Next, via a union bound and Lemma \ref{lemma:sort-bound}, with probability at least $1 - \delta / 3$, we have
\[
| \theta_{T_{[1]} } - \theta_{( \bigcup_{j, |\hat{B}_j| \geq 2} \hat{B}_j )_{[1]} } | \leq \eps.
\]
Note that $S = T \cup R$. Hence, this lemma is proved.
\end{proof}

Lemma~\ref{lemma:HE-sample-comp} proves the third item of Theorem~\ref{thm:HE}. Before proceeding to the lemma, we first introduce the following statement.

\begin{lemma} \label{lemma:n-mean-est-sample}
Let $Z_i$ denote the sample complexity for $\MeanEst(i, \eps, \delta)$. With probability at least $1 - \delta / 3$, $\sum_{i \in S} Z_i$ is bounded by $O \left( \sum_{i \in S} \left( \frac{\sigma_i^2}{\eps^2} + \frac{1}{\eps} \right) ( \ln \delta^{-1} + \ln \ln \eps^{-1} )  \right)$.
\end{lemma}

\begin{proof}
Define $T_i = \left( \frac{\sigma_i^2}{\eps^2} + \frac{1}{\eps} \right) ( \ln \delta^{-1} + \ln \ln \eps^{-1} )$.  According to Lemma \ref{LEMMA:MEAN-EST} and Lemma \ref{LEMMA:MEAN-EST-COROL}, we can find a constant $c > 0$ such that $\Pr[ Z_i \leq cT_i] \geq 1 - \delta$ and $\Pr[ Z_i > jcT_i] \leq \delta \cdot 2^{-20j}$ for $j \geq 3$. Define $\widetilde{Z}_i = (Z_i - 3cT_i) \cdot \mathds{1} \{ Z_i \geq 3cT_i \} $. Hence, we have $\widetilde{Z_i} \geq 0$ and 
\begin{align*}
\E[ \widetilde{Z}_i ] \leq \sum_{j \geq 1} \Pr [ \widetilde{Z}_i \in ( (j -1)cT_i, jcT_i ] ] \cdot j c T_i 
 \leq \sum_{j \geq 1} \delta \cdot 2^{-20j} \cdot j c T_i = \left( \sum_{j \geq 1} 2^{-20j} \cdot j \right) \delta c T_i \leq \delta c T_i.
\end{align*}
Note that $\widetilde{Z}_i \geq Z_i - 3cT_i$. Therefore, $\sum_{i} \widetilde{Z}_i \geq \sum_{i} Z_i - 3c \sum_{i} T_i$.
Further by applying Markov's Inequality, we can get 
\begin{align*}
\Pr \left[ \sum_{i} Z_i \geq 6c \sum_i T_i \right] \leq \Pr \left[ \sum_{i} \widetilde{Z}_i \geq 3c\sum_i T_i \right]  
\leq \frac{ \E\left[  \sum_i \widetilde{Z}_i \right]  }{ 3c \sum_i T_i } = \frac{ \sum_i \E[   \widetilde{Z}_i ]  }{ 3c\sum_i T_i } \leq \frac{\delta c \sum_i T_i }{3c\sum_i T_i} = \delta / 3,
\end{align*}
which means with probability at least $1 - \delta / 3$, $\sum_{i \in S} Z_i$ is bounded by  
\begin{equation*}
O\left( \sum_i T_i \right) = O \left( \sum_{i \in S} \left( \frac{\sigma_i^2}{\eps^2} + \frac{1}{\eps} \right) ( \ln \delta^{-1} + \ln \ln \eps^{-1} )  \right).
\end{equation*}
\end{proof}

\begin{lemma} \label{lemma:HE-sample-comp}
With probability at least $1 - \delta / 3$, the sample complexity of $\HE(S, \eps, \delta)$ is \\ $O\left( \sum_{i \in S} \left( \frac{ \sigma_i^2}{\eps^2} + \frac{1}{\eps} \right) ( \ln \delta^{-1} + \ln \ln \eps^{-1} ) \right)$.
\end{lemma}

\begin{proof}
According to Lemma~\ref{LEMMA:VARIANCE-EST:a}, the sample complexity of $\VarEst(i, \delta / (2 N^2 ), \eps)$ is bounded by \\ $O( \frac{1}{\eps} \ln (\delta / (2N^2)) ) = O( \frac{1}{\eps} (\ln \delta^{-1} + \ln \ln \eps^{-1}))$. Hence, the number of samples used at Line~3 is 
\begin{equation} \label{equ:HE-sample-comp-1}
O\left( \frac{|S|}{\eps} (\ln \delta^{-1} + \ln \ln \eps^{-1}) \right).
\end{equation}

According to Lemma~\ref{lemma:n-mean-est-sample}, with probability at least $\delta / (27N) \geq 1 - \delta / 3$, the number of samples used at Line~8 is 
\begin{equation} \label{equ:HE-sample-comp-2}
O \left( \sum_{i \in S} \left( \frac{\sigma_i^2}{\eps^2} + \frac{1}{\eps} \right) ( \ln \delta^{-1} + \ln \ln \eps^{-1} )  \right).
\end{equation}

Combining (\ref{equ:HE-sample-comp-1}) and (\ref{equ:HE-sample-comp-2}), we prove this lemma.
\end{proof}

\subsection{Analysis of the Procedure $\ME$}\label{sec:proof-thm-ME}
In this subsection, we prove Theorem~\ref{thm:ME} as follows.
\begin{myproof}{Proof of Theorem~\ref{thm:ME}}
Let $\calF_r$ denote the event that the claim in Theorem \ref{thm:HE} holds for $\HE(T_r, \eps_r, \delta_r)$. Hence, $\Pr[ \calF_r ] \geq 1 - \delta_r$. Let $\calF = \bigwedge_{r = 1}^{+\infty} \calF_r $. Via a union bound, we can see $\Pr[ \calF ] \geq 1 - \sum_{r = 1}^{+\infty} \delta_r \geq 1 - \delta$.

Conditioning on event $\calF$ which happens with probability at least $1 - \delta$, we will show all three items in the theorem statement hold. 

For the first item, note that $|T_{r+1}| \leq \frac{2}{3} |T_r|$ by Lemma~\ref{LEMMA:HE-CARD-CUT}. Hence, there are at most $O( \ln |T_0| ) = O( \ln |S| )$ rounds. Suppose when algorithm $\ME(S, \eps, \delta)$ terminates, $r = r_0$. Also note that $|R^{r+1}| = O( \ln \eps_r^{-1} )$ by Theorem~\ref{thm:HE:a}. Therefore, $|R_{r_0}| = O( \ln |S| \ln \eps_{r_0-1}^{-1} ) = O( (\ln |S|)^2 \ln \eps^{-1} )$. Together with $|T_{r_0}| \leq 10$, the first item is proved.

For the second item, for each round, we have $| \theta_{(T_{r+1} \cup R^{r+1})_{[1]} } - \theta_{ (T_r)_{[1]} }  |  \leq \eps_r $ by $\calF$. Since $R_{r+1} = R_r \cup R^{r+1}$, we get $| \theta_{(T_{r+1} \cup R_{r+1})_{[1]} } - \theta_{(T_r \cup R_r)_{[1]} }  |  \leq \eps_r $. Therefore, we obtain
\begin{multline*}
| \theta_{T_{[1]} } - \theta_{S_{[1]} } | = | \theta_{(T_{r_0} \cup R_{r_0})_{[1]} } - \theta_{(T_0 \cup R_0)_{[1]} }  |
\leq \sum_{r = 0}^{r_0 - 1} | \theta_{(T_{r + 1} \cup R_{r + 1})_{[1]} } - \theta_{(T_r \cup R_r)_{[1]} }  | = \sum_{r= 0}^{r_0 - 1} \eps_r = \eps (1 - \beta) \sum_{r= 0}^{r_0 - 1} \beta^r  \leq \eps,
\end{multline*}
which concludes the proof of the second item.

Now we come to the third item. By $\calF$, we have $\sum_{i \in T_{r + 1}} (\sigma_i^2 + \eps_{r+1}) \leq \frac{255}{256} \sum_{i \in T_r} (\sigma_i^2 + \eps_r)$ for $r = 0, 1, \dots, r_0 - 1$. Therefore, the sample complexity is 
\begin{multline*}
\sum_{r = 0}^{r_0 - 1} O \left( \frac{ \sum_{i \in T_r} (\sigma_i^2 + \eps_r) }{\eps_r^2} ( \ln \delta_r^{-1} + \ln \ln \eps_r^{-1} )  \right) \\
= \sum_{r = 0}^{r_0 - 1} O \left( \frac{ \sum_{i \in T_0} (\sigma_i^2 + \eps) }{\eps^2} \left( \frac{255/256}{ \beta^2 } \right)^r ( \ln \delta_r^{-1} + \ln \ln \eps_r^{-1} )  \right) \\
= O \left( \sum_{i \in S} \left( \frac{\sigma_i^2}{\eps^2} + \frac{1}{\eps} \right) ( \ln \delta^{-1} + \ln \ln \eps^{-1} )  \right),
\end{multline*}
which concludes the proof of the third item. 
\end{myproof}

\section{Missing Proofs in Section~\ref{SEC:N-ARM}} \label{app:proofs-in-sec:n-arm}

\subsection{Proof of Lemma~\ref{LEMMA:BEST-ARM-REMAINS}}

Assume the best arm is not eliminated before round $r \geq 1$ begins i.e., $S_{[1]} \in S_r$. According to Lemma~\ref{LEMMA:MEAN-EST}, with probability at least $1 - \delta_r / 18$, there is 
\begin{equation} \label{equ-1:best-arm-remains}
 | \hat{\theta}_{i}^r - \theta_i | \leq \eps_r/2
\end{equation}
for every arm $i$.
Via a union bound, with probability at least $1 - \delta_r / 9$, it holds that 
\[
\hat{\theta}_{a_r}^r \leq \theta_{a_r} + \eps_r / 2 \leq \theta_{ S_{[1]} } + \eps_r / 2 \leq \hat{\theta}_{ S_{[1]} }^r + \eps_r,
\]
where both the first and last inequalities are due to \eqref{equ-1:best-arm-remains}, which means the best arm $S_{[1]}$ will not be eliminated during round $r$.

Since $S_{[1]} \in S_1$, we obtain with probability at least $\prod_{r = 1}^{+\infty} (1 - \delta_r / 9) \geq 1 - \sum_{r=1}^{+\infty} \delta_r / 9 = 1 - \delta / 9 \sum_{r=1}^{+\infty} 1 / (2r^2) \geq 1 - \delta / 9$, $\calM_1$ holds concluding the proof of this lemma.

\subsection{Proof of Lemma~\ref{LEMMA:BEST-ARM-OUTPUT}}

Throughout this proof we condition on the event $\calM_1$.

By $\calM_1$, we know that the best arm is never eliminated. Hence, there is $|S_r| > 1$ when $r \leq \left\lceil \log_2 \Delta_2^{-1} \right\rceil + 1$.

By Lemma~\ref{LEMMA:MEAN-EST}, with probability at least $1 - \delta_r / 18$, there is 
\begin{equation}\label{equ-1:best-arm-output}
 | \hat{\theta}_{a_r}^r - \theta_{a_r}^r | \leq \eps_r/2 
\end{equation} 
and with probability at least $1 - \delta_r / 18$, there is 
\begin{equation}\label{equ-2:best-arm-output}
| \hat{\theta}_{a_r^*}^r - \theta_{a_r^*} | \leq \eps_r/2.
\end{equation} 
By Theorem \ref{thm:OptArmEst}, with probability at least $1 - \delta_r / 18$, there is 
\begin{equation}\label{equ-3:best-arm-output}
| \theta_{a_r} - \theta_{ (S_r)_{[1]} } | \leq \eps_r / 2,
\end{equation}
and with probability at least $1 - \delta_r / 18$, there is 
\begin{equation} \label{equ-4:best-arm-output}
| \theta_{a_r^*} - \theta_{(S_r \backslash \{ a_r \})_{[1]} } | \leq \eps_r / 2.
\end{equation}
Let $\calM_2^\prime$ denote the event that \eqref{equ-1:best-arm-output}, \eqref{equ-2:best-arm-output}, \eqref{equ-3:best-arm-output} and \eqref{equ-4:best-arm-output} hold for all round $r$. By a union bound, we have $\Pr[ \calM_2^\prime ] \geq 1- \sum_{r = 1}^{+\infty} 2\delta_r / 9 = 1- 2\delta / 9 \sum_{r = 1}^{+\infty} 1 / (2r^2) \geq 1- 2\delta / 9$.

Conditioning on event $\calM_2^\prime$ which happens with probability at least $1 - 2\delta / 9$, we will show that $\calM_2$ holds.

First we claim that the algorithm must terminate at round $r$ where $ r \leq r_0 = \left\lceil \log_2 \Delta_2^{-1} \right\rceil + 1$.
Note that when $r = r_0$, we have $\eps_r = \eps_{r_0} \leq \Delta_2 / 8$. Also by $\calM_1$, we have $(S_r)_{[1]} = S_{[1]} $.  Then according to (\ref{equ-3:best-arm-output}), we can get $\theta_{a_r} \geq \theta_{S_{[1]} } - \eps_r /2 > \theta_{S_{[2]} } $, which means $a_r = S_{[1]}$. Further by (\ref{equ-4:best-arm-output}), we can also get $\theta_{a_r^*} \leq \theta_{S_{[2]} } + \eps_r / 2$. Hence, when $r = r_0$, it holds that 
\begin{multline*}
| \hat{\theta}_{a_r} - \hat{\theta}_{a_r^*} | \geq \hat{\theta}_{a_r} - \hat{\theta}_{a_r^*} \geq \theta_{a_r} - \theta_{a_r^*} - \eps_r 
\geq \theta_{S_{[1]} }  - \theta_{a_r^*} - 3/2 \eps_r \geq \theta_{S_{[1]} }  - \theta_{S_{[2]} } - 2 \eps_r = \Delta_2 - 2\eps_r > 2\eps_r,
\end{multline*}
where the second inequality is due to (\ref{equ-1:best-arm-output}) and (\ref{equ-2:best-arm-output}). This means that the algorithm must terminate when $r = r_0$. Hence, this claim is proved.

Next we claim that when the algorithm terminates, it holds that $a_r = S_{[1]} $. Suppose not, by (\ref{equ-4:best-arm-output}), we have $| \theta_{a_r^*} - \theta_{S_{[1]} } | \leq \eps_r/2 $. Hence, 
\[
| \hat{\theta}_{a_r} - \hat{\theta}_{a_r^*} | \leq | \hat{\theta}_{a_r}^r - \theta_{a_r}^r | + | \theta_{a_r}^r - \theta_{S_{[1]} } | + | \theta_{S_{[1]} } - \theta_{a_r^*} | + | \theta_{a_r^*} - \hat{\theta}_{a_r^*} | \leq 2\eps_r,
\]
which is a contradiction since the only termination criteria is $| \hat{\theta}_{a_r} - \hat{\theta}_{a_r^*} | > 2\eps_r$.

\subsection{Proof of Lemma~\ref{LEMMA:ELIM-PROB}}
Note that $\Delta_i > 2^{-s} \geq 2^{r} = 4\eps_r$ when $i \in S_{r-1,s}$ and  $ r \geq s$.

By Lemma \ref{LEMMA:MEAN-EST}, with probability at least $1 - \delta_r / 18$, we have 
\begin{equation}\label{equ-1:elim-prob}
 | \hat{\theta}_{i}^r - \theta_i | \leq \eps_r/2 
\end{equation} 
and with probability at least $1 - \delta_r / 18$, we have 
\begin{equation}\label{equ-2:elim-prob}
| \hat{\theta}_{a_r}^r - \theta_{a_r} | \leq \eps_r/2.
\end{equation} 
By Theorem \ref{thm:OptArmEst}, with probability at least $1 - \delta_r / 18$, we have 
\begin{equation}\label{equ-3:elim-prob}
| \theta_{a_r} - \theta_{S_{[1]} } | \leq \eps_r / 2.
\end{equation}
Via a union bound, we can get with probability at least $1 - \delta_r / 4$, it holds that 
\[
\hat{\theta}_i^r \leq \theta_i + \eps_r /2 = \theta_1 - \Delta_i + \eps_r / 2 \leq \theta_{a_r} + \eps_r / 2 - \Delta_i + \eps_r / 2 \leq \hat{\theta}_{a_r}^r + 3/2 \eps_r - \Delta_i < \hat{\theta}_{a_r}^r - \eps_r,
\]
where the first inequality is due to (\ref{equ-1:elim-prob}), the third inequality is due to (\ref{equ-3:elim-prob}) and the second last inequality is due to (\ref{equ-2:elim-prob}).

\subsection{Proof of Lemma~\ref{LEMMA:SAMP-PER-ARM}}

By Lemma~\ref{LEMMA:ELIM-PROB}, the probability that arm $i \in A_s$ is eliminated at round $r \geq s + 1$ is at most 
\[
\prod_{i = s}^{r - 1} \delta_i / 4 \leq (\delta_1 / 4) ^{r - s - 1} = (\delta / 8) ^{r - s - 1}.
\] Hence, with probability at least $1 - (\delta / 8)^{r - s - 1}$, $\sum_{r = 1}^{+\infty} I_i^r T_i^r$ is bounded by
\begin{multline*}
\sum_{t = 1}^{r} T_i^t
= \sum_{t = 1}^{r} O \left( \left( \frac{\sigma_i^2}{ \eps_t^2 } + \frac{1}{\eps_t} \right) ( \ln \delta^{-1}  + \ln \ln \eps^{-1} + \ln r ) \right) \\
= O \left( \left( \frac{\sigma_i^2}{ \eps_{r}^2 } + \frac{1}{\eps_{r}} \right) ( \ln \delta^{-1}  + \ln \ln \eps^{-1} + \ln r ) \right) 
= O \left( 4^{r - s - 1} \left( \frac{\sigma_i^2}{ \Delta_i^2 } + \frac{1}{\Delta_i} \right) ( \ln \delta^{-1} + \ln \ln \Delta_i^{-1} ) \right),
\end{multline*}
where the last inequality is due to $\Delta_i = \Theta( \eps_s )$. 

\subsection{Proof of Lemma~\ref{LEMMA:SAMP-BY-SUB-OPTIMAL-ARMS}}

Let $Z_i = \sum_{r = 1}^{+\infty} I_i^r T_i^r$. Here, we only consider the arms indexed with $i \neq S_{[1]} $. Define $T_i = \Big( \Big( \frac{\sigma_i^2}{\Delta_i^2} + \frac{1}{\Delta_i} \Big) ( \ln \delta^{-1} + \ln \ln \Delta_i^{-1} ) \Big)$.  According to Lemma~\ref{LEMMA:SAMP-PER-ARM}, we can find a constant $c > 0$ such that $\Pr[ Z_i \leq 4^jcT_i] \geq 1 - (\delta/8)^j$. Define $\widetilde{Z}_i = (Z_i - 4cT_i) \cdot \mathds{1} \{ Z_i \geq 4cT_i \} $. Hence, we have $\widetilde{Z_i} \geq 0$ and 
\begin{align*}
\E[ \widetilde{Z}_i ] \leq \sum_{j \geq 1} \Pr [ \widetilde{Z}_i \in ( 4^{j -1}cT_i, 4^jcT_i ] ] \cdot 4^j c T_i  \leq \sum_{j \geq 1} (\delta / 8)^j \cdot 4^j c T_i = \left( \sum_{j \geq 1} (\delta /2)^j \right) c T_i \leq \delta c T_i.
\end{align*}
Note that $\widetilde{Z}_i \geq Z_i - 4cT_i$ which implies $\sum_{i} \widetilde{Z}_i \geq \sum_{i} Z_i - 4c \sum_{i} T_i$.
Applying Markov's Inequality, we obtain 
\begin{align*}
\Pr \left[ \sum_{i} Z_i \geq 22c \sum_i T_i \right] \leq \Pr \left[ \sum_{i} \widetilde{Z}_i \geq 18c\sum_i T_i \right]
\leq \frac{ \E\left[  \sum_i \widetilde{Z}_i \right]  }{ 18c \sum_i T_i } = \frac{ \sum_i \E[   \widetilde{Z}_i ]  }{ 18c\sum_i T_i } \leq \frac{\delta c \sum_i T_i }{18c\sum_i T_i} = \delta / 18.
\end{align*}
Therefore, with probability at least $1 - \delta / 18$, $\sum_{i \neq S_{[1]} } Z_i$ is bounded by  
\begin{equation*}
O\left( \sum_{i \neq S_{[1]} } T_i \right) = O \left( \sum_{i \neq S_{[1]} } \left( \frac{\sigma_i^2}{\Delta_i^2} + \frac{1}{\Delta_i} \right) ( \ln \delta^{-1} + \ln \ln \Delta_i^{-1} )  \right).
\end{equation*}

\subsection{Proof of Lemma~\ref{LEMMA:SAMP-BY-OPTIMAL-ARMS}}
By $\calM_2$, the algorithm terminates with $r = O( \ln \Delta_2^{-1}) = O( \ln \Delta_1^{-1})$.
Hence, $\sum_{r = 1}^{+\infty} I_{S_{[1]} }^r T_{S_{[1]} }^r$ is bounded by
\begin{multline*}
\sum_{t = 1}^{r} O \left( \left( \frac{\sigma_i^2}{ \eps_t^2 } + \frac{1}{\eps_t} \right) ( \ln \delta^{-1}  + \ln \ln \eps^{-1} + \ln r ) \right) \\
= O \left( \left( \frac{\sigma_i^2}{ \eps_{r}^2 } + \frac{1}{\eps_{r}} \right) ( \ln \delta^{-1}  + \ln \ln \eps^{-1} + \ln r ) \right) 
= O \left( \left( \frac{\sigma_i^2}{ \Delta_1^2 } + \frac{1}{\Delta_1} \right) ( \ln \delta^{-1} + \ln \ln \Delta_1^{-1} ) \right),
\end{multline*}
where the last inequality is due to $\Delta_1 = O( \eps_r )$. 

\subsection{Proof of Lemma~\ref{LEMMA:OPT-ARM-ID-ROUND-SAMP}}

We analyze the sample complexity in round $r$ as follows. According to Lemma \ref{lemma:n-mean-est-sample}, with probability at least $1 - \delta_r / 54$, Line~5 uses $O ( \sum_{i \in S_r} ( \frac{\sigma_i^2}{\eps_r^2} + \frac{1}{\eps_r} ) ( \ln \delta_r^{-1} + \ln \ln \eps_r^{-1} ) ) $ samples. By Theorem~\ref{thm:OptArmEst} and a union bound, with probability at least $1 - \delta_r / 9$, Line~6  and Line~7 use $O ( \sum_{i \in S_r} ( \frac{\sigma_i^2}{\eps_r^2} + \frac{1}{\eps_r} ) ( \ln \delta_r^{-1} + \ln \ln \eps_r^{-1} ) ) $ samples. Via a union bound, we obtain that with probability at least $1 - \delta_r / 54 - \delta_r / 9 \geq 1 -  \delta_r / 6$, the sample complexity of round $r$ is 
\[
O \left( \sum_{i \in S_r} \left( \frac{\sigma_i^2}{\eps_r^2} + \frac{1}{\eps_r} \right) ( \ln \delta_r^{-1} + \ln \ln \eps_r^{-1} ) \right) = \sum_{i \in S} O( I_i^r T_i^r ).
\]

Applying a union bound over all rounds, we have that with probability at least $1 - \sum_{r = 1}^{+\infty} \delta_r /6  = 1 -   \delta / 6 \sum_{r = 1}^{+\infty} 1 / (2r^2) \geq 1- \delta / 6 $, for each $r$, the sample complexity of round $r$ is $\sum_{i \in S} O( I_i^r T_i^r )$.

\section{The Lower Bound} \label{sec:LB}
Before presenting the lower bound, we would like to introduce some notations. Let $I = \{X_1, \dots, X_n\}$ denote the input instance where $X_i$ represents the random reward when arm $i$ is sampled. With a little abuse of notations, we let $I_{[i]}$ to be the index of the $i$-th best arm in $I$. For any best arm identification algorithm $\bbA$, and any input instance $I$, let $T^{\bbA}(I)$ and $T_i^{\bbA}(I)$ be the random variables denoting the numbers of samples made by $\bbA$ on input $I$ and arm $i$ respectively. When it is clear from the context, we usually omit the superscript. Let us denote by $\mathcal{I}_{\sigma_1^2, \dots,  \sigma_n^2, \Delta_2, \dots, \Delta_n}$ the set of instances where the $i$-th best arm has variance $\sigma_i^2$ and for $i \geq 2$, the gap between the $i$-th best arm and the best arm is $\Delta_i$. Our goal of this section is to prove the following theorem. 

\begin{theorem} \label{thm:lower-bound-n}
For any $\sigma_i^2 < 0.1, i \in [n]$  and $0 < \Delta_i < 0.1, i = 2, \dots, n$, there exists an instance $I \in \mathcal{I}_{\sigma_1^2, \dots,  \sigma_n^2, \Delta_2, \dots, \Delta_n}$ such that
for any $\delta$-correct best arm identification algorithm $\mathbb{A}_{\delta}$ ($\delta < 0.1$), there is  
\[
\frac{ \E  [ T(I) ] }{ \sum_{i = 1}^n \left( \frac{\sigma_i^2}{\Delta_i^2} + \frac{1}{\Delta_i} \right) \ln \delta^{-1}  } \geq \frac{1}{80},
\]
where $T(I) = T^{\mathbb{A}_{\delta}}(I)$ is the number of samples used by $\mathbb{A}_\delta$.
\end{theorem}

To prove the theorem, given $\{ \sigma_i \}_{i \in [n]}$ and $\{ \Delta_i \}_{i = 2, \dots, n}$, we create an instance $I \in \mathcal{I}_{\sigma_1^2, \dots,  \sigma_n^2, \Delta_2, \dots, \Delta_n}$. Note that $\E[ T(I) ] = \sum_{i = 1}^n \E[ T_i(I) ]$ where $T_i(I) = T_i^{\bbA_{\delta}}(I)$ is the number of samples used by $\mathbb{A}_\delta$ on arm $i$. We utilize the Change of Distribution lemma (Lemma~\ref{lemma:change-of-distribution}) to bound every $\E[ T_i(I) ]$ separately. 
In order to bound $\E[ T_1(I) ]$, we create new instances similar to $I$ where the best arm in $I$ becomes the second best. To deal with the upper bound of $\E[ T_i(I) ]$ for $i \geq 2$, we create different new instances where the $i$-th best arm in $I$ becomes the best arm.  

More specifically, for any fixed $\sigma_i^2 < 0.1, i\in [n]$ and $0 < \Delta_i < 0.1, i = 2, \dots, n$, we consider the following instance $I = \{ X_1, \dots, X_n\}$ where 
$$
X_1 = \left\{
 \begin{array}{ll}  
    0.5 + \sigma_1, & \textrm{w.p.}\ 0.5  \\ 
    0.5 - \sigma_1,  & \textrm{w.p.}\  0.5  \\ 
  \end{array}
  \right. \text{and}
$$
$$
 X_i = \left\{
 \begin{array}{ll}  
    0.5 - \Delta_i + \sigma_i, & \textrm{w.p.}\ 0.5  \\ 
    0.5 - \Delta_i - \sigma_i,  & \textrm{w.p.}\ 0.5 \\ 
  \end{array}
  \right. \textrm{for} \ i \geq 2.
$$
It can be easily verified that $I \in \mathcal{I}_{\sigma_1^2, \dots,  \sigma_n^2, \Delta_2, \dots, \Delta_n}$. For this instance, we have the following two lemmas (Lemma~\ref{LEMMA:LOWER-BOUND-S1} and Lemma~\ref{LEMMA:LOWER-BOUND-S2}), among which Lemma~\ref{LEMMA:LOWER-BOUND-S1} gives a lower bound on $\E[ T_1[I] ]$ and Lemma~\ref{LEMMA:LOWER-BOUND-S2} gives a lower bound on $\E[ T_i[I] ]$ for $i \geq 2$. We defer the proof of these two lemmas to Section~$\ref{sec:lower-bound-s1}$ and Section~$\ref{sec:lower-bound-s2}$ respectively.

\begin{lemma} \label{LEMMA:LOWER-BOUND-S1}
$$
\frac{ \E [ T_1(I) ] } { \left( \frac{\sigma_1^2}{\Delta_2^2} + \frac{1}{\Delta_2} \right) \ln \delta^{-1} } \geq \frac{1}{80}.
$$
\end{lemma}

\begin{lemma} \label{LEMMA:LOWER-BOUND-S2}
For any $i \geq 2$, it holds that
$$
\frac{ \E [ T_i(I) ] } { \left( \frac{\sigma_i^2}{\Delta_i^2} + \frac{1}{\Delta_i } \right) \ln \delta^{-1} } \geq \frac{1}{30}.
$$
\end{lemma}

With these two lemmas, we are ready to prove Theorem~\ref{thm:lower-bound-n}.

\begin{myproof}{Proof of Theorem~\ref{thm:lower-bound-n}}
By Lemma~\ref{LEMMA:LOWER-BOUND-S1}, Lemma~\ref{LEMMA:LOWER-BOUND-S2}, and noting that $\Delta_1 = \Delta_2$, we have 
\[
\frac{ \E [ T(I) ] }{ \sum_{i = 1}^n \left( \frac{\sigma_i^2}{\Delta_i^2} + \frac{1}{\Delta_i} \right) \ln \delta^{-1}  } \geq \min_{i \in [n]} \left\{ \frac{ \E [ T_i(I) ] } { \left( \frac{\sigma_i^2}{\Delta_i^2} + \frac{1}{\Delta_i } \right) \ln \delta^{-1} } \right\} \geq \frac{1}{80} .
\]
\end{myproof}

\subsection{Proof of Lemma~\ref{LEMMA:LOWER-BOUND-S1}} \label{sec:lower-bound-s1}
We prove the lemma under two different scenarios.

\paragraph{Case 1: $\sigma_1^2 \geq 5\Delta_2$.} Since $\sigma_1 \leq 1$, we have $\sigma_1 \geq 5\Delta_2$, Consider the following instance $I'_1 = \{ X_1', \dots, X_n'\}$ where 
$$
X_1' = \left\{
 \begin{array}{ll}  
    0.5 + \sigma_1, & \textrm{w.p.}\ 0.5 - \Delta_2 / \sigma_1  \\ 
    0.5 - \sigma_1,  & \textrm{w.p.}\  0.5 + \Delta_2 / \sigma_1 \\ 
  \end{array}
  \right.
$$
and $ X_i' = X_i $ for $i \geq 2$. Since $I_{[1]} = 1 \neq 2 = (I'_1)_{[1]}$, applying Lemma~\ref{lemma:change-of-distribution}, we have 
\begin{align*}
\E[ T_1[I] ] & \geq \frac{1}{\KL(X_1, X_1')} \ln \frac{1}{2.4 \delta } \\
& = \frac{2}{\ln (1 - 4 \Delta_2^2 / \sigma_1^2 )^{-1}} \ln \frac{1}{2.4\delta } \\
 & \geq \frac{1 - 4 \Delta_2^2 / \sigma_1^2 }{2 \Delta_2^2 / \sigma_1^2} \ln \frac{1}{e\delta } \\
 & \geq \frac{1}{5} \cdot \frac{\sigma_1^2}{\Delta_2^2} \ln \delta^{-1},
\end{align*}
where the second last inequality is due to $\ln (1 + x) \leq x$ for $x > -1$ and the last inequality is due to $\Delta_2 / \sigma_1 \leq 0.2$ and $\delta < 0.1$. Hence, there is
\begin{equation} \label{equ-1:lower-bound-n-s1}
\frac{ \E[ T_1[I] ] } { \left( \frac{\sigma_1^2}{\Delta_2^2} + \frac{1}{\Delta_2 } \right) \ln \delta^{-1} } \geq \frac{ \E[ T_1[I] ] } { \frac{6}{5} \cdot \frac{\sigma_1^2}{\Delta_2^2} \ln \delta^{-1} } \geq \frac{ \frac{1}{5} \cdot \frac{\sigma_1^2}{\Delta_2^2} \ln \delta^{-1} }{ \frac{6}{5} \cdot \frac{\sigma_1^2}{\Delta_2^2} \ln \delta^{-1}} = \frac{1}{6},
\end{equation}
where the first inequality is due to $\sigma_1^2 \geq 5 \Delta_2$.

\paragraph{Case 2: $\sigma_1^2 < 5\Delta_2$.} Consider the following instance $I''_1 = \{ X_1'', \dots, X_n'' \}$ where
$$
X_1'' = \left\{
 \begin{array}{ll}  
    0.5 + \sigma_1, & \textrm{w.p.}\ 0.5 - 2\Delta_2   \\ 
    0.5 - \sigma_1,  & \textrm{w.p.}\  0.5 - 2\Delta_2 \\
    0, & \textrm{w.p.}\ 4\Delta_2 
  \end{array}
  \right.,
$$
and $X_i'' = X_i $ for $i \geq 2$. Since $I_{[1]} =1 \neq 2 =  (I''_1)_{[1]} $, applying Lemma~\ref{lemma:change-of-distribution}, we have 
\begin{align*}
\E[ T_1[I] ] & \geq \frac{1}{\KL(X_1, X_1'')} \ln \frac{1}{2.4 \delta } \\
& = \frac{1}{\ln (1 - 4 \Delta_2 )^{-1}} \ln \frac{1}{2.4\delta } \\ 
&\geq \frac{1 - 4 \Delta_2 }{4 \Delta_2 } \ln \frac{1}{e\delta } \\
& \geq \frac{3}{40} \cdot \frac{1}{\Delta_2} \ln \delta^{-1},
\end{align*}
where the second last inequality is due to $\ln (1 + x) \leq x$ for $x > -1$ and the last inequality is due to $\Delta_2 \leq 0.1$ and $\delta < 0.1$. Hence, there is
\begin{equation} \label{equ-2:lower-bound-n-s1}
\frac{ \E[ T_1[I] ]  } { \left( \frac{\sigma_1^2}{\Delta_2^2} + \frac{1}{\Delta_2 } \right) \ln \delta^{-1} } \geq \frac{ \E[ T_1[I] ] } { 6 \cdot \frac{1}{\Delta_2} \ln \delta^{-1} } \geq \frac{ \frac{3}{40} \cdot \frac{1}{\Delta_2} \ln \delta^{-1} }{ 6 \cdot \frac{1}{\Delta_2} \ln \delta^{-1} } = \frac{1}{80},
\end{equation}
where the first inequality is due to $\sigma_1^2 < 5 \Delta_2$. 

Combining (\ref{equ-1:lower-bound-n-s1}) and (\ref{equ-2:lower-bound-n-s1}), we prove this lemma.  

\subsection{Proof of Lemma~\ref{LEMMA:LOWER-BOUND-S2}} \label{sec:lower-bound-s2}

The idea is the same as that used for bounding $\E[ T_1[I] ]$. However, we need to construct slightly different instances. 

Let $i$ be any fixed integer satisfying $2 \leq i \leq n$. Similarly, we prove the lemma under two different scenarios.

\paragraph{Case 1: $\sigma_i^2 \geq 5\Delta_i$.} Since $\sigma_i \leq 1$, we have $\sigma_i \geq 5\Delta_i$.  Consider the following instance $I'_i = \{ X_1', \dots, X_n' \}$ where
$$
X_i' = \left\{
 \begin{array}{ll}  
    0.5 - \Delta_i + \sigma_i, & \textrm{w.p.}\ 0.5 + \Delta_i / \sigma_i  \\ 
    0.5 - \Delta_i - \sigma_i,  & \textrm{w.p.}\  0.5 - \Delta_i / \sigma_i \\ 
  \end{array}
  \right.
$$
and $ X_j' = X_j $ for $j \neq i$. Since $I_{[1]} = 1 \neq i = (I'_i)_{[1]}$, applying Lemma~\ref{lemma:change-of-distribution}, we get 
\begin{align*}
\E[ T_i[I] ] & \geq \frac{1}{\KL(X_i, X_i')} \ln \frac{1}{2.4 \delta } \\
& = \frac{2}{\ln (1 -  4\Delta_i^2 / \sigma_i^2 )^{-1}} \ln \frac{1}{2.4\delta } \\
& \geq \frac{ 1 -  4\Delta_i^2 / \sigma_i^2 }{ 2\Delta_i^2 / \sigma_i^2} \ln \frac{1}{e\delta } \\
& \geq \frac{1}{5} \cdot \frac{\sigma_i^2}{\Delta_i^2} \ln \delta^{-1},
\end{align*}
where the second last inequality is due to $\ln (1 + x) \leq x$ for $x > -1$ and the last inequality is due to $\Delta_i / \sigma_i \leq 0.2$ and $\delta < 0.1$. Hence, there is 
\begin{equation} \label{equ-1:lower-bound-n-s2}
\frac{ \E[ T_i[I] ] } { \left( \frac{\sigma_i^2}{\Delta_i^2} + \frac{1}{\Delta_i } \right) \ln \delta^{-1} } \geq \frac{ \E[ T_i[I] ] } { \frac{6}{5} \cdot \frac{\sigma_i^2}{\Delta_i^2} \ln \delta^{-1} } \geq \frac{ \frac{1}{5} \cdot \frac{\sigma_i^2}{\Delta_i^2} \ln \delta^{-1} }{ \frac{6}{5} \cdot \frac{\sigma_i^2}{\Delta_i^2} \ln \delta^{-1}} = \frac{1}{6},
\end{equation}
where the first inequality is due to $\sigma_i^2 \geq 5 \Delta_i$.

\paragraph{Case 2: $\sigma_i^2 < 5\Delta_i$.} Consider the following instance $I''_i = \{ X_1'', \dots, X_n'' \}$ where
$$
X_i'' = \left\{
 \begin{array}{ll} 
    1, & \textrm{w.p.} \ 2\Delta_i \\
    0.5 - \Delta_i + \sigma_i, & \textrm{w.p.}\ 0.5 - \Delta_i   \\ 
    0.5 - \Delta_i - \sigma_i,  & \textrm{w.p.}\  0.5 - \Delta_i  \\
  \end{array}
  \right.
$$
and $X_j'' = X_j $ for $j \neq i$. Since $I_{[1]} = 1 \neq i = (I''_i)_{[1]}$, applying Lemma~\ref{lemma:change-of-distribution}, we have 
\begin{align*}
\E [ T_i(I) ] & \geq \frac{1}{\KL(X_i, X_i'')} \ln \frac{1}{2.4 \delta } \\
& = \frac{1}{\ln (1 - 2 \Delta_i )^{-1}} \ln \frac{1}{2.4\delta } \\
& \geq \frac{1 - 2 \Delta_i }{2 \Delta_i } \ln \frac{1}{e\delta } \\
& \geq \frac{1}{5} \cdot \frac{1}{\Delta_i} \ln \delta^{-1},
\end{align*}
where the second last inequality is due to $\ln (1 + x) \leq x$ for $x > -1$ and the last inequality is due to $\Delta_i \leq 0.1$ and $\delta < 0.1$. Hence, there is 
\begin{equation} \label{equ-2:lower-bound-n-s2}
\frac{ \E [ T_i(I) ] } { \left( \frac{\sigma_i^2}{\Delta_i^2} + \frac{1}{\Delta_i } \right) \ln \delta^{-1} } \geq \frac{ \E [ T_i(I) ] } { 6 \cdot \frac{1}{\Delta_i} \ln \delta^{-1} } \geq \frac{ \frac{1}{5} \cdot \frac{1}{\Delta_i} \ln \delta^{-1} }{ 6 \cdot \frac{1}{\Delta_i} \ln \delta^{-1} } = \frac{1}{30},
\end{equation}
where the first inequality is due to $\sigma_i^2 < 5 \Delta_i$. 

Combining (\ref{equ-1:lower-bound-n-s2}) and (\ref{equ-2:lower-bound-n-s2}), we prove this lemma.

\section{Change of Distribution Lemma}

\begin{lemma}[A special case of Lemma~1 in \cite{kaufmann2016complexity}] \label{lemma:change-of-distribution}
Given two multi-armed bandit instances $I = \{ X_1, \dots, X_n\}$ and $I' = \{ X_1', \dots, X_n' \}$ such that $I_{[1]} \neq I'_{[1]}$, for any $\delta$-correct algorithm $\bbA$, it holds that 
$$
\sum_{i = 1}^{n} \E[ T_i^{\bbA}(I) ] \KL(X_i, X_i') \geq \ln \frac{1}{2.4 \delta}.
$$  
\end{lemma}

\end{document}